\documentclass[11pt]{article}

\usepackage{fullpage,times,url,bm}

\usepackage{amsthm,amsfonts,amsmath,amssymb,epsfig,color,float,graphicx,verbatim}
\usepackage{algorithm,algorithmic}
\usepackage{bbm}
\usepackage{natbib}

\usepackage{pgfplots}
\usepackage{graphicx}
\usepackage{subcaption}
\usepackage[normalem]{ulem}
\usepackage{authblk}
\usepackage{array}

\usepackage{hyperref}
\hypersetup{
	colorlinks   = true, 
	urlcolor     = blue, 
	linkcolor    = blue, 
	citecolor   = black 
}

\newtheorem{theorem}{Theorem}[section]
\newtheorem{proposition}{Proposition}[section]
\newtheorem{lemma}{Lemma}[section]
\newtheorem{corollary}{Corollary}[section]
\newtheorem{definition}{Definition}[section]
\newtheorem{remark}{Remark}[section]
\newtheorem{example}{Example}

\newcommand{\secref}[1]{Section~\ref{#1}}
\newcommand{\subsecref}[1]{Subsection~\ref{#1}}
\newcommand{\figref}[1]{Figure~\ref{#1}}
\renewcommand{\eqref}[1]{Eq.~(\ref{#1})}
\newcommand{\lemref}[1]{Lemma~\ref{#1}}

\newcommand{\thmref}[1]{Theorem~\ref{#1}}
\newcommand{\propref}[1]{Proposition~\ref{#1}}
\newcommand{\appref}[1]{Appendix~\ref{#1}}

\newcommand{\defref}[1]{Def.~\ref{#1}}
\newcommand{\itemref}[1]{Item~\ref{#1}}
\newcommand{\asmref}[1]{Assumption~\ref{#1}}
\newcommand{\corref}[1]{Corollary~\ref{#1}}
\newcommand{\remref}[1]{Remark~\ref{#1}}

\newcommand{\Lcal}{\mathcal{L}}
\newcommand{\Ocal}{\mathcal{O}}

\newcommand{\Dcal}{\mathcal{D}}

\newcommand{\Ncal}{\mathcal{N}}

\newcommand{\one}[1]{\mathbbm{1}\left\{#1\right\}}
\newcommand{\abs}[1]{\left|#1\right|}
\newcommand{\p}[1]{\left(#1\right)}
\newcommand{\pcc}[1]{\left[#1\right]}
\newcommand{\set}[1]{\left\{#1\right\}}
\newcommand{\relu}[1]{\sigma\left(#1\right)}

\usepackage{amssymb}

\usepackage{bbm}
\newcommand{\onefunc}{\mathbbm{1}}

\newcommand{\stam}[1]{}

\usepackage{mathtools}

\newtheorem{assumption}[theorem]{Assumption}

\newcommand{\bx}{\mathbf{x}}
\newcommand{\bw}{\mathbf{w}}

\newcommand{\bb}{\mathbf{b}}
\newcommand{\bu}{\mathbf{u}}
\newcommand{\bv}{\mathbf{v}}
\newcommand{\bz}{\mathbf{z}}

\newcommand{\bd}{\mathbf{d}}

\newcommand{\bh}{\mathbf{h}}

\newcommand{\balpha}{\boldsymbol{\alpha}}

\newcommand{\bxi}{\boldsymbol{\xi}}

\newcommand{\btheta}{{\boldsymbol{\theta}}}

\newcommand{\co}{{\cal O}}

\newcommand{\cd}{{\cal D}}

\newcommand{\ci}{{\cal I}}

\newcommand{\cl}{{\cal L}}

\newcommand{\cn}{{\cal N}}

\newcommand{\pr}{\mathbb{P}}

\DeclareMathOperator*{\sign}{sign}

\DeclareMathOperator*{\E}{\mathbb{E}}

\DeclareMathOperator{\erf}{erf}

\newcommand{\reals}{{\mathbb R}}

\newcommand{\zero}{{\mathbf{0}}}
\newcommand{\diag}{\mathrm{diag}}
\newcommand{\inner}[1]{\langle #1 \rangle}
\newcommand{\norm}[1]{\left\|#1\right\|}
\newcommand{\snorm}[1]{\|#1\|} 
\newcommand{\spnorm}[1]{\left\|#1\right\|_{\text{sp}}} 

\makeatletter
\newcommand{\printfnsymbol}[1]{%
  \textsuperscript{\@fnsymbol{#1}}%
}
\makeatother

\title{On the Effective Number of Linear Regions in Shallow Univariate ReLU Networks: Convergence Guarantees and Implicit Bias}

\ifdefined\usebigfont

\usepackage{times}
\usepackage[fontsize=13pt]{scrextend}
\usepackage[left=1.56in,right=1.56in,top=1.71in,bottom=1.77in]{geometry}
\usepackage[shortlabels]{enumitem}
\setlist[itemize]{leftmargin=*}
\setlist[enumerate]{leftmargin=*}
\else
\fi

\begin{document}

\author[1]{Itay Safran\thanks{Equal contribution}}
\author[2]{Gal Vardi\printfnsymbol{1}}
\author[1]{Jason D.\ Lee}
\affil[1]{Princeton University}
\affil[2]{TTI-Chicago and Hebrew University}

\date{}

\maketitle

\begin{abstract}
We study the dynamics and implicit bias of gradient flow (GF) on univariate ReLU neural networks with a single hidden layer in a binary classification setting. We show that when the labels are determined by the sign of a target network with $r$ neurons, with high probability over the initialization of the network and the sampling of the dataset, GF converges in direction (suitably defined) to a network achieving perfect training accuracy and having at most $\mathcal{O}(r)$ linear regions, implying a generalization bound. Unlike many other results in the literature, under an additional assumption on the distribution of the data, our result holds even for mild over-parameterization, where the width is $\tilde{\mathcal{O}}(r)$ and independent of the sample size.
\end{abstract}

\section{Introduction}

Over-parameterized neural networks are known to be easier to train compared to their smaller counterparts, despite the resulting increase in the problem's dimensionality and the required computational resources \citep{daniely2017sgd,allen2018convergence,safran2018spurious,du2018gradient,du2019gradient,ji2019neural,zou2020gradient,li2020learning,safran2021effects,zhou2021local}. However what is perhaps more surprising, is that in stark contrast to our classic understanding of generalization in machine learning models, this does not seem to degrade the generalization capabilities of the learned model in spite of the significant increase in its capacity. It is widely believed that what plays a key role in explaining this phenomenon is what is commonly referred to in the literature as \emph{implicit bias/regularization} \citep{neyshabur2014search,zhang2021understanding}, where the optimization algorithm used inadvertently exhibits a bias towards empirical minimizers with a certain property that might induce better generalization. For example, such properties may include having a small norm or quasi-norm of the weights (e.g.\ \citet{neyshabur2017implicit,neyshabur2017exploring,lyu2019gradient,woodworth2020kernel,ji2020directional}) or a low rank solution (e.g.\ \citet{razin2020implicit}).

In this work, we study the dynamics and the implicit bias of GF on univariate ReLU neural networks, with the underlying assumption that the labels are determined by the sign of a target network of width $r$ (which thus changes sign between $-1$ and $1$ at most $r$ times). Our analysis reveals that under the assumption of an i.i.d., normally-distributed initialization of the weights and biases of the network, over-parameterization (i.e.\ the use of width strictly larger than $r$) is necessary for attaining a small population loss. Moreover, if we also assume that the width scales at least linearly (up to logarithmic terms, excluding dependence on the confidence parameter) with the length of the shortest interval on which the labels do not change sign, then with high probability over the initialization of the network and the sampling of the data, over-parameterization is sufficient for driving the empirical loss to be small enough so that all data instances are classified correctly. Thereafter, by analyzing the implicit bias of GF 
as the time $t$ tends to infinity, 
we show that we converge in direction (see \secref{sec:preliminaries} for a formal definition) to a network having at most $\Ocal(r)$ linear regions. Since the minimal number of neurons required to express a network with an arbitrarily small loss in general is $r$, this demonstrates that the implicit bias of GF in our setting is such that optimization converges to a solution which effectively has the optimal number of neurons up to a constant factor, which also provides a clear geometric interpretation with an intuitive generalization bound. This is in contrast to norm-based results where the characterization of the implicit bias in function space is less clear. Overall, our analysis provides an end-to-end result on the learnability of univariate ReLU neural networks with respect to GF in a binary classification setting.

The remainder of this paper is structured as follows: After specifying our contributions in more detail below, we turn to discuss related work. In \secref{sec:preliminaries}, we present our notations and assumptions used throughout the paper before we present our main theorem. In \secref{sec:optimization} we formally present our optimization result. In \secref{sec:implicit_bias} we turn to analyze the implicit bias of GF in our setting. Finally, in \secref{sec:generalization} we show that the implicit bias leads to a generalization bound.


\subsection*{Our contributions}

\begin{itemize}
    \item
    We prove that when training a sufficiently wide network which is initialized appropriately (\asmref{asm:init}) and under a suitable assumption on the distribution of the data (\asmref{asm:data}), then with high probability there exists some time $t_0$ where GF attains training error at most $\frac{1}{2n}$ on a size-$n$ sample (\thmref{thm:optimization}). Our width requirement depends on the length of the shortest interval where the classificiation does not change sign, which in certain cases requires only mild over-parameterization with far fewer parameters in the model compared to observations in the dataset (see Remark~\ref{rmk:mild vs extreme}).
    
    \item
    We show that if GF achieves training error smaller than $\frac{1}{n}$ at some time $t_0$, then it converges to zero loss and converges in direction (see \secref{sec:preliminaries} for a formal definition) to a network with at most $\Ocal(r)$ linear regions (\thmref{thm:minimize regions}). Since this result provides a simple and intuitive geometric interpretation to the implicit bias of GF in our setting, it readily translates to a generalization bound using a standard argument (\corref{cor:small_loss_generalizes}). We point out that unlike our optimization guarantee, this result holds regardless of the initialization of the network.
    
    
    \item
    Finally, we combine our optimization and generalization results to derive our main theorem (\thmref{thm:main}), which establishes an end-to-end learnability result for GF in our setting. This indicates that a wide model facilitates optimization, yet at the same time implicit bias prevents us from overfitting, even if the architecture we train is far wider than what is necessary. The result holds in the \emph{rich regime}, and thus provides a guarantee which goes beyond the analysis achieved using NTK-based results (see \remref{rem:rich}).
    
    \item
    As an additional contribution, we show (under 
    our assumption on the initialization) that width at least $1.3r$ is necessary for GF to attain population loss less than an absolute constant over a particular target function (\thmref{thm:dormant_neuron}). Along with our previous results, this shows that over-parameterization is not just sufficient, but also necessary for successful learning.
\end{itemize}

We now turn to discuss some of the related work in the literature which is most relevant to ours in more detail.

\subsection*{Related work}

\paragraph{Implicit bias in neural networks.}

The literature on the implicit bias in neural networks has rapidly expanded in recent years, and cannot be reasonably surveyed here (see \cite{vardi2022implicit} for a survey). In what follows, we discuss only results which apply to depth-$2$ ReLU networks.

By \cite{lyu2019gradient,ji2020directional} homogeneous neural networks (and specifically depth-$2$ ReLU networks) trained with exponentially-tailed classification losses converge in direction to a KKT point of the maximum-margin problem. 
Our analysis of the implicit bias relies on this result.
We note that the aforementioned KKT point may not be a global optimum (see a discussion in \secref{sec:implicit_bias}). 
For depth-$2$ ReLU networks trained with the square loss there are no known guarantees on the implicit bias (cf.\ \cite{vardi2021implicit,timor2022implicit}).

Several works in recent years studied the implication of minimizing the $\ell_2$ norm of the weights on the function space in depth-$2$ univariate ReLU networks. In \cite{savarese2019infinite} and \cite{ergen2021convex} it is shown that a minimal-norm fit for a sample is given by the linear spline interpolation (i.e., a ``connect-the-dots" function). In such linear spline interpolation the number of linear regions is small.
The former work considered only regression, while the latter considered both regression and classification.
Note that margin-maximization is equivalent to norm-minimization with margin at least $1$. Thus, our result can also be viewed as an analysis of the implication of the bias towards norm-minimization on the learned function.
We emphasize three important differences between the results from \cite{savarese2019infinite,ergen2021convex} and ours:
\begin{enumerate}
\item As we already mentioned, the result of \cite{lyu2019gradient,ji2020directional} implies a certain bias towards margin maximization, but it does not guarantee convergence to a global optimum (or even to a local optimum) of the maximum-margin problem. The only guarantee is that GF converges to a KKT point. Our result relies only on convergence to such a KKT point, and (unlike \cite{savarese2019infinite} and \cite{ergen2021convex}) it does not assume convergence to a global optimum. As a result, we are able to obtain provable generalization bounds for GF. 
\item In \cite{savarese2019infinite} and \cite{ergen2021convex} it is shown that the linear spline interpolation minimizes the weights' norms. However, they also show that it is not a unique minimizer. Thus, in addition to the linear spline interpolation there are also other networks that fit the training set and minimize the norms. Therefore, even under the assumption that the weights' norms are minimized, their results do not guarantee convergence to  a function with a small number of linear regions (as in our result).
\item \cite{savarese2019infinite} and \cite{ergen2021convex} consider norm-minimization of the weights without the bias terms, while the implicit bias towards margin-maximization due to \cite{lyu2019gradient,ji2020directional} (which we rely on) is w.r.t.\ all the parameters, including the bias terms. Hence, the implicit bias in depth-$2$ ReLU networks with exponentially-tailed losses does not minimize the norms in the sense considered in \cite{savarese2019infinite,ergen2021convex}. 
\end{enumerate}
We note that \cite{ergen2021revealing} showed that linear spline interpolators minimize the norms also in deep univariate networks.
The result from \cite{savarese2019infinite} was extended to multi-variate functions in \cite{ongie2019function}.
\cite{parhi2020role} studied the relation between norm minimization and spline interpolation for a broader family of activation functions. \cite{hanin2021ridgeless} gave a geometric characterization of all depth-$2$ univariate ReLU networks with a single linear unit, that minimize the $\ell_2$ norm of the weights (excluding bias terms) and interpolate a given dataset (in a regression setting).
\cite{blanc2020implicit} studied the relation between the implicit bias of SGD in depth-$2$ univariate ReLU networks (in a regression setting) and the number of convexity changes of the learned network.
\cite{maennel2018gradient} showed that for a given training dataset there
are only finitely many functions that GF with small initialization may converge to in depth-$2$ ReLU networks, independent of the network size.
\cite{chizat2020implicit} studied the dynamics of GF on infinite width depth-$2$ networks with exponentially-tailed losses and showed bias towards margin maximization w.r.t. a certain function norm known as the variation norm.
\cite{phuong2020inductive} studied the implicit bias in depth-$2$ ReLU networks trained on orthogonally separable data.


\paragraph{Convergence of gradient methods under extreme over-parameterization.}

In recent years, many theoretical works have focused on providing convergence guarantees for training depth-$2$ neural networks with non-linear activations. 
 \citet{andoni2014learning} provide a convergence guarantee for learning polynomials of degree $r$ in $d$-dimensional space in a regression setting using networks of width roughly $d^{2r}$. Since their architecture excludes bias terms which can be simulated by incrementing the input dimension by $1$, their result in fact requires width $2^{2r}$ in our univariate setting, whereas our width requirement is typically much milder. Following the success of the NTK \citep{jacot2018neural}, a spate of papers provided convergence guarantees when training using GD (e.g.\ \citep{allen2018convergence,du2018gradient,du2019gradient,ji2019neural,zou2020gradient}). The main difference that sets our work apart is that our width requirement is given in terms of the complexity of the teacher network, irrespective of the sample size $n$, whereas these works require that the width scales polynomially with $n$, which could be significantly larger. Moreover, as mentioned earlier, such results operate in the lazy regime where the features that are learned are dictated mainly by the initialization rather than the training process, whereas our analysis enters the rich regime once the loss becomes sufficiently small, and provides a result that goes beyond NTK-based analyses. On the flip side, our analysis only holds for binary classification in the one-dimensional setting. Similarly to us, \citet{soltanolkotabi2018theoretical} provide convergence guarantees by establishing that the objective function satisfies the PL-condition (see \citet{polyak1963gradient}), however unlike our optimization guarantee and similarly to previously discussed works, their result requires that the network has more trainable parameters than data instances. \citet{chizat2019lazy} establish that by scaling a model appropriately, we can effectively interpolate between the lazy and the rich regime. While we use this observation in our optimization result to drive the loss to become sufficiently small in the first stage of optimization, our implicit bias result nevertheless operates in the rich regime regardless of this scaling.

\paragraph{Teacher-student setting and mild over-parameterization.}

In this paper, we assume that the labels of the data are determined by the sign of a teacher network of width $r$. Such a similar teacher-student setting but for a regression problem allowed the study of mild over-parameterization in quite a few recent works. \citet{safran2018spurious,arjevani2020analytic,arjevani2021analytic} show the existence of spurious (non-global) local minima in the loss landscape in this setting. Other works provide certain recovery guarantees; assuming absolute value activations, \citet{li2020learning} provide a global convergence guarantee to loss at most $o(1/r)$, and \citet{safran2021effects,zhou2021local} provide local convergence guarantees for ReLU or absolute value activations.
These works require width at least $\text{poly}(r)$ for convergence, whereas in our setting we show that width $\tilde{\Ocal}(r)$ suffices in certain cases, and that over-parameterization is also necessary for successful optimization under our assumptions. While our results might superficially seem to provide stronger guarantees, we stress that the seemingly stronger bounds we derive are made possible in part due to the different assumptions made which include a univariate domain with biases compared to a multivariate domain with no biases nor output layer weights, thus highlighting the difference between the two architectures. In light of this, we argue that the bounds in these results are not directly comparable to ours.

\section{Preliminaries and main result} \label{sec:preliminaries}

\paragraph{Notations.}

We use bold-face letters to denote vectors, e.g., $\bx=(x_1,\ldots,x_d)$. For $\bx \in \reals^d$ we denote by $\norm{\bx}$ the Euclidean norm.
We denote by $\onefunc[\cdot]$ the indicator function, for example $\onefunc[t \geq 5]$ equals $1$ if $t \geq 5$ and $0$ otherwise. We denote 
$\sign(z) = 1$ if $z>0$ and $-1$ otherwise.
For an integer $d \geq 1$ we denote $[d]=\{1,\ldots,d\}$. 
We use standard asymptotic notation $\co(\cdot)$ to hide constant factors. A function $f:D\to\reals$ which is twice continuously differentiable in a domain $D\subseteq\reals^d$ is said to satisfy the \emph{PL-condition} if there exists $\lambda>0$ such that $\frac12\norm{\nabla f(\bx)}^2\ge\lambda(f(\bx)-f^*)$, where $f^*\coloneqq\inf_{\bx}f(\bx)$.

\paragraph{Neural networks.}

The ReLU activation function is defined by $\sigma(z) = \max\{0,z\}$.
In this work we consider depth-$2$ ReLU neural networks with input dimension $1$. Formally, a depth-$2$ network $\cn_\btheta$ of width $k$ is parameterized by $\btheta = [\bw, \bb, \bv]$ where $\bw,\bb,\bv \in \reals^k$, and for every input $x \in \reals$ we have 
\begin{equation}\label{eq:architecture}
	\cn_\btheta(x) = \sum_{j \in [k]}v_j \sigma(w_j \cdot x + b_j)~. 
\end{equation}
We sometimes view $\btheta$ as the vector obtained by concatenating the vectors $\bw, \bb, \bv$. Thus, $\norm{\btheta}$ denotes the $\ell_2$ norm of the vector $\btheta$.
We denote $\Phi(\btheta; x) := \cn_\btheta(x)$.
Given a network $\Ncal_\btheta(x)$ as above, we refer to the set of its non-differentiable points (w.r.t.\ the variable $x$) as its breakpoints.

\paragraph{Gradient flow (GF) and implicit bias.}

Let $S = \{(x_i,y_i)\}_{i=1}^n \subseteq \reals \times \{-1,1\}$ be a binary classification training dataset. Let $\Phi(\btheta; \cdot):\reals \to \reals$ be a neural network parameterized by $\btheta$. 
For a loss function $\ell:\reals \to \reals$ the \emph{empirical loss} of $\Phi(\btheta; \cdot)$ on the dataset $S$ is 
\begin{equation}
\label{eq:objective}
	\cl(\btheta) := \frac{1}{n} \sum_{i=1}^n \ell(y_i \Phi(\btheta; x_i))~.
\end{equation} 
We focus on the exponential loss $\ell(q) = e^{-q}$ and the logistic loss $\ell(q) = \log(1+e^{-q})$.

We consider GF on the objective given in \eqref{eq:objective}. This setting captures the behavior of GD with an infinitesimally small step size. Let $\btheta(t)$ be the trajectory of GF. Starting from an initial point $\btheta(0)$, the dynamics of $\btheta(t)$ are given by the differential equation 
$\frac{d \btheta(t)}{dt} \in -\partial^\circ \cl(\btheta(t))$. Here, $\partial^\circ$ denotes the \emph{Clarke subdifferential}, which is a generalization of the derivative for non-differentiable functions (see Appendix~\ref{app:KKT} for a formal definition).
We say that a trajectory $\btheta(t)$ {\em converges in direction} to $\btheta^*$ if 
$\lim_{t \to \infty}\frac{\btheta(t)}{\norm{\btheta(t)}} = 
\frac{\btheta^*}{\norm{\btheta^*}}$.

\subsection{Assumptions}

Our main result holds under the following assumption on the initialization of the weights of the network.

\begin{assumption}[Network initialization]\label{asm:init}~
    \begin{itemize}
        \item
        The weights and biases $w_i,b_i$ of each neuron $i\in[k]$ in the hidden layer are i.i.d.\ and satisfy
        \[
            w_i,b_i\sim\Ncal(0,\sigma_{\text{h}}^2).
        \]
        \item
        The weights $v_i$, $i\in[k]$ of the output neuron are i.i.d.\ and satisfy
        \[
            v_i\sim\Ncal(0,\sigma_{\text{o}}^2).
        \]
    \end{itemize}
\end{assumption}

We point out that our particular choice of normally distributed weights is not essential, and that our results will also hold for example under the assumption of uniformly distributed weights, but with a slightly different proof and constants in the resulting bounds. To facilitate our analysis, we make the following assumptions on the distribution of the data and its corresponding labels.

\begin{assumption}[Data distribution]\label{asm:data}
    There exist a natural $r\ge1$ and real $R\ge1$ and $C,\rho>0$ such that following hold:
    
    \begin{itemize}
        \item
        There exists a depth-$2$ ReLU network $\Ncal^*$ of width $r$ such that the examples $(x,y)$ of the data satisfy $x\sim\Dcal$ and $y=\sign(\Ncal^*(x))$.
        \item 
        The density of $\Dcal$ denoted by $\mu$ satisfies $\mu(x)=0$ for all $x\notin[-R,R]$.
        \item
        $\sup_{x}\mu(x)\le C$.
        \item
        $\rho>0$ is the length of the shortest interval $I\subseteq[-R,R]$ such that $\sign(\Ncal^*(x))$ is the same for all $x\in I$, and for all $I\subset I'$, there exist $x,x'\in I'$ such that $\sign(\Ncal^*(x))\neq\sign(\Ncal^*(x'))$.
    \end{itemize}
\end{assumption}

We remark that the above assumption is mostly mild. The main non-trivial requirement is that the distribution is compactly-supported, however the last two assumptions always hold for some $C,\rho>0$ if $\mu$ is continuous on $\reals$ for example, and any target function that changes sign at most $r$ times can be expressed by a network of width $r$.\footnote{We also remark that the assumption $R\ge1$ is only for simplicity of presentation, since our results also hold for any $R<1$ with the same bounds we get when plugging $R=1$.}

\subsection{Main result}

Having stated our assumptions, we now turn to present our main theorem in this paper.

\begin{theorem}\label{thm:main}
    Under Assumptions~\ref{asm:init}~and~\ref{asm:data}, given any $\varepsilon,\delta\in(0,1)$, suppose that the following hold
    \begin{align*}
    	&n \geq C_0 \cdot \frac{r \log(1/\varepsilon) + \log(2/\delta)}{\varepsilon}, &
        & k \ge 6144\cdot\frac{ R^4\log\p{\frac{48r}{\delta}}}{\rho},\\ &\sigma_{\text{h}}\ge 8400\cdot\frac{n^2C R^{3.5}\sqrt{rk}}{\delta\rho},  & &\sigma_{\text{o}}\le\frac{1}{4kR\sigma_{\text{h}}\log\p{\frac{12k}{\delta}}},
    \end{align*}
    where $C_0>0$ is a universal constant. Then with probability at least $1-\delta$ over the randomness in the initialization of the network and the sampling of a size-$n$ dataset, GF converges to zero loss, and converges in direction to $\btheta^*$ such that the network $\cn_{\btheta^*}$ has at most $32 r + 67$ linear regions and satisfies
    \[
        \pr_{x \sim \cd}\left[\sign(\cn_{\btheta^*}(x))\neq \sign(\cn^*(x)) \right] \leq \varepsilon.
    \]
\end{theorem}

Our theorem is a result of breaking the proof into two different stages and combining them using a simple union bound. Specifically, at the first stage we use \thmref{thm:optimization}, which establishes that the empirical loss drops below $\frac1n$ with high probability; and in the second stage we use \corref{cor:small_loss_generalizes} to argue that the implicit bias takes effect once the empirical loss is sufficiently small which results in a generalization bound. The theorem suggests two interesting implications: (i)
We may train an arbitrarily wide network without risk of overfitting since the implicit bias of GF dictates that we converge to a model with low capacity;\footnote{Note that in our univariate setting we can deduce $r$ by analyzing the data and possibly estimate the minimal required width for attaining small training error using our derived bounds, however in more general settings (e.g.\ the multivariate case) $r$ may not be easily deduced from the dataset, which might prompt us to train the widest network possible given available computational resources.} and (ii), we further gain a sample complexity bound which is independent of 
$\rho$ and $R$.
We also note that only the direction of $\btheta$ affects the classification and the number of linear regions in $\cn_\btheta$, and that the scale of $\btheta$ is not important here. Namely, for every $\btheta$, $x$ and every $\alpha>0$ we have $\cn_{\alpha \btheta}(x) = \alpha^2 \cn_\btheta(x)$, and thus the scale of $\btheta$ affects only the scale of the outputs of $\cn_\btheta$ and not the classification nor the partition to linear regions and therefore nor does it affect the generalization properties of $\btheta$.

Lastly,  we remark that our initialization scheme used in \thmref{thm:main} is somewhat unorthodox, in the sense that we require the hidden layer to have a rather large variance which scales polynomially with the sample size. While in practice it is more common that the weights in the hidden layer have a smaller variance (e.g.\ \citet{glorot2010understanding,he2015delving}), our scaling prevents the breakpoints of the neurons in the trained network to move too much before we are able to decrease the training error sufficiently, and the impact of the magnitude by which we scale the hidden layer upon initialization on the dynamics of GF was studied in a similar univariate regression setting \citep{williams2019gradient,sahs2020shallow}. We now conclude the discussion of our main result with the following remarks on the setting studied in our paper.
\begin{remark}[Mild vs.\ extreme over-parameterization]\label{rmk:mild vs extreme}
    In this paper, we make a distinction between what we call the \emph{mild over-parameterization} regime, where the required width of the network $k$ scales with $r$ but not with the sample size $n$; and the \emph{extreme over-parameterization} regime, where the 
    width $k'$ of the network exceeds $n$. Under this distinction, for sufficiently large $n$, we will always have that $k\ll n< k'$, and thus $k\ll k'$.
\end{remark}

\begin{remark}[GF vs.\ GD]\label{rmk:gf vs gd}
    It is important to stress that our results hold for GF which ignores computational considerations and does not necessarily imply the convergence of GD. For this reason, it is of utmost importance to generalize our results to hold for GD rather than just GF. That being said, there is some recent evidence suggesting that at least in certain cases, positive results on GF may be translated to GD \citep{elkabetz2021continuous}. Moreover, at the very least, \thmref{thm:optimization} can indeed be generalized to hold for GD (see discussion after the theorem statement). In any case, for the sake of coherence we focus in this paper on GF, and we leave generalizations for GD as an important future work direction.
\end{remark}

\begin{remark}[Rich vs.\ lazy regime] \label{rem:rich}
    Our optimization analysis operates in the lazy regime where the hidden layer does not move much and most of the learning is performed in the output neuron. However, once our analysis goes into the second stage where the implicit bias takes effect, we essentially move into the rich (aka the feature-learning) regime, where redundant features (i.e.\ excess neurons) are being effectively discarded at the limit $t\to\infty$. In light of this, as was discussed in the related work section, our result provides a guarantee which goes beyond the analysis achieved using NTK-based results.
\end{remark}

\section{Over-parameterization leads to small empirical loss}\label{sec:optimization}

In this section, we analyze the dynamics of GF on the objective defined in \eqref{eq:objective}. Our main contribution is to establish that sufficient over-parameterization guarantees that GF leads to a point with empirical loss which is sufficiently small. 
Formally, we present the following theorem.

\begin{theorem}\label{thm:optimization}
    Under Assumptions~\ref{asm:init}~and~\ref{asm:data}, given any $\delta\in(0,1)$, suppose that the following hold
    \begin{equation}\label{eq:convergence_assumptions}
        k \ge 6144\cdot\frac{ R^4\log\p{\frac{24r}{\delta}}}{\rho}, \hskip 0.4cm \sigma_{\text{h}}\ge 4200\cdot\frac{n^2C R^{3.5}\sqrt{rk}}{\delta\rho} \hskip 0.4cm \text{and}\hskip 0.4cm  \sigma_{\text{o}}\le\frac{1}{4kR\sigma_{\text{h}}\log\p{\frac{6k}{\delta}}}.
    \end{equation}
    Then with probability at least $1-\delta$ over the randomness in the initialization of the network and the sampling of a size-$n$ dataset, there exists time $t_0$ such that GF initialized from $\btheta(0)$ reaches a point $\btheta(t_0)$ satisfying $\Lcal(\btheta(t_0)) \le \frac{1}{2n}$.
    
\end{theorem}

The above theorem essentially requires that we use width which is proportional to $1/\rho$ up to logarithmic factors to facilitate optimization. If $\rho=\Omega(1/r)$, then this requires that we over-parameterize by a multiplicative constant up to logarithmic factors. We remark that our result can also be adapted to hold for GD rather than GF with a polynomial number of iterations.\footnote{To show this, one would need to bound the length of the trajectory of GD for objectives that satisfy the PL-condition locally. See \appref{app:proof_optimization} for further detail.} In any case, as discussed in \remref{rmk:gf vs gd}, we stress that our focus here is to show that GF attains sufficiently small loss so that our implicit bias analysis takes effect, and we leave generalizations for GD and milder width requirements for future work.

The proof of the above theorem, which appears in \appref{app:optimization_proof}, relies on over-parameterizing sufficiently to the extent of having at least three breakpoints on each constant segment where the data does not change classification, and four additional neurons that are active on all the data instances. The key observation is that under such over-parameterization, we can identify a direction in weight space which moves the current network configuration in a manner which strictly decreases the objective value. This allows us to establish that the objective function satisfies the PL-condition locally in a neighborhood around our initialization. Finally, by bounding the length of the trajectory of GF, we show that the objective value decreases to $\frac{1}{2n}$ before we can escape the neighborhood in which the PL-condition is satisfied.

Interestingly, \thmref{thm:optimization} already implies a generalization bound if the sample size is sufficiently larger than the degrees of freedom in the student network. Nevertheless, such an approach alone is not capable of obtaining a sample complexity which is independent of $\rho$ and $R$ (since the width of the student network and thus also its capacity scale with these parameters, implying a generalization bound that explicitly depends on them). Moreover, understanding the implicit bias of GF is of independent interest, even if we ignore the improvement it provides to the sample complexity.

It is natural to explore what is the minimal amount of over-parameterization required for attaining a small loss in our setting. A modest requirement is that we are able to make the generalization error arbitrarily small given a sufficiently large sample. It is thus interesting to present the following theorem, which establishes that under \asmref{asm:init}, GF is not capable of attaining loss below an absolute constant unless $k\ge\lfloor1.3r\rfloor$, regardless of the sample size. 

    \begin{theorem}\label{thm:dormant_neuron}
        Define the population loss of a network with weights $\btheta$ w.r.t.\ a distribution $\Dcal$ and teacher network $\Ncal^*(\cdot)$ as $\Lcal_{\Dcal}(\btheta)\coloneqq \E_{x\sim\Dcal}\pcc{\ell(\Ncal_{\btheta}(x)\cdot \sign(\Ncal^*(x)))}$. Let $\alpha\ge1$ and suppose that an architecture as in \eqref{eq:architecture}, having width $k=\lfloor\alpha r\rfloor$ is initialized according to \asmref{asm:init}. Then there exists a distribution $\Dcal$ and $\Ncal^*(\cdot)$ of width $r$ such that for any time $t\ge0$ and any $n\ge1$, for a size-$n$ sample drawn from $\Dcal$ and labeled by $\sign(\Ncal^*(\cdot))$, with probability at least $0.25$ over the initialization of the network, GF trained on \eqref{eq:objective} attains population loss at least
        \[
            \Lcal_{\Dcal}(\btheta(t))\ge\frac14\p{1-0.75\alpha}.
        \]
    \end{theorem}

Thus, under the theorem's assumptions, over-parameterization by a constant factor is required. We refer the reader to \appref{app:necessary} for the construction of $\Dcal$ and $\Ncal^*(\cdot)$, some further discussion, and the full proof of this lower bound.

\section{The implicit bias of GF}\label{sec:implicit_bias}

In this section, we show that GF converges to networks where the number of linear regions is minimal up to a constant factor.
We first give some required background and discuss an important result on the implicit bias which applies to depth-$2$ ReLU networks, and then state our result.

\subsection{Required background}

The following theorem gives an important characterization of the implicit bias in depth-$2$ ReLU networks:
\begin{theorem}[\cite{lyu2019gradient,ji2020directional}] \label{thm:known KKT}
	Let $\Phi(\btheta; \cdot)$ be a depth-$2$ ReLU neural network parameterized by $\btheta$. Consider minimizing either the exponential or the logistic loss over a binary classification dataset $ \{(x_i,y_i)\}_{i=1}^n$ using GF. Assume that there exists time $t_0$ such that $\cl(\btheta(t_0)) < \frac{1}{n}$, namely, $y_i \Phi(\btheta(t_0); x_i) > 0$ for every $x_i$. Then, GF converges in direction to a first order stationary point (KKT point) of the following maximum margin problem in parameter space:
\begin{equation} \label{eq:optimization problem}
	\min_\btheta \frac{1}{2} \norm{\btheta}^2 \;\;\;\; \text{s.t. } \;\;\; \forall i \in [n] \;\; y_i \Phi(\btheta; x_i) \geq 1~.
\end{equation}
Moreover, $\cl(\btheta(t)) \to 0$ and $\norm{\btheta(t)} \to \infty$ as $t \to \infty$.
\end{theorem}

We note that the above theorem holds for the more general case of homogeneous neural networks (in parameter space), but for this work it suffices to consider depth-$2$ networks, which are indeed homogeneous.
Note that in ReLU networks Problem~(\ref{eq:optimization problem}) is non-smooth. Hence, the KKT conditions are defined using the Clarke subdifferential.
See Appendix~\ref{app:KKT} for more details of the KKT conditions. 
\thmref{thm:known KKT} characterized the implicit bias of GF with the exponential and the logistic losses for depth-$2$ ReLU networks. Namely, even though there are many possible directions $\frac{\btheta}{\norm{\btheta}}$ that classify the dataset correctly, GF converges only to directions that are KKT points of Problem~(\ref{eq:optimization problem}).
We note that such a KKT point is not necessarily a global/local optimum (cf.\ \cite{vardi2021margin}). Thus, under the theorem's assumptions, GF \emph{may not} converge to an optimum of Problem~(\ref{eq:optimization problem}), but it is guaranteed to converge to a KKT point. This is demonstrated in the following example for the case of depth-$2$ univariate networks, which is our focus.

\begin{example} \label{ex:not global}
	Let $\cn_\btheta$ be a depth-$2$ univariate network of width $2$, namely, $\cn_\btheta(x) = v_1 \sigma(w_1 x + b_1) + v_2 \sigma(w_2 x + b_2)$. Let $S = \{(x_1,y_1), (x_2,y_2)\}$ be a size-$2$ dataset such that $x_1=4$, $x_2 = -4$ and $y_1=y_2=1$. Suppose that we train $\cn_\btheta$ on the dataset $S$ using GF with the exponential or the logistic loss, and that the initialization $\btheta(0)$ is such that $b_1=v_1=1$ and $w_1=w_2=b_2=v_2=0$. Note that $\cl(\btheta(0)) = \frac{1}{2} \cdot 2 \ell(1) < \frac{1}{2}$ (for both the exponential and the logistic loss). Hence, by \thmref{thm:known KKT} GF converges to zero loss, and converges in direction to a KKT point $\btheta^*$ of Problem~(\ref{eq:optimization problem}). By observing the gradient $\nabla_\btheta \cl(\btheta)$ it is not hard to show that for every time $t$ we have $w_2(t)=b_2(t)=v_2(t)=0$, namely, the second neuron remains inactive, and we have $\btheta^* = \btheta(0)$ (see Appendix~\ref{app:example details} for details). However, $\btheta^*$ is not a global optimum of Problem~(\ref{eq:optimization problem}). Indeed, consider $\btheta'$ such that $w'_1=v'_1=v'_2=\frac{1}{2}$, $w'_2=-\frac{1}{2}$, and $b'_1=b'_2=0$. Then, $y_i \cn_{\btheta'}(x_i) = 1$ for all $i \in \{1,2\}$, and we have $\snorm{\btheta'}^2 = 1 < 2 = \snorm{\btheta^*}^2$.
\end{example}

The above example implies that although GF has a certain bias towards margin maximization (as shown in \thmref{thm:known KKT}), it may not maximize the margin. Hence, we cannot obtain margin-based generalization bounds based on the bias towards margin maximization.

\subsection{Characterization of the implicit bias}

We now state our main result on the implicit bias:

\begin{theorem} \label{thm:minimize regions}
	Let $S = \{(x_i,y_i)\}_{i=1}^n \subseteq \reals \times \{-1,1\}$ be a dataset such that $x_1 < \ldots < x_n$, and for all $i \in [n]$ we have $y_i \cn^*(x_i) > 0$, where $\cn^*:\reals \to \reals$ is a depth-$2$ ReLU network of width $r$. Note that $|\{i \in [n-1] : y_i \neq y_{i+1} \}| \leq r$.
	Consider 
	GF on a depth-$2$ neural network $\cn_\btheta$ w.r.t. the dataset $S$.
	Assume that there exists time $t_0$ such that $\cl(\btheta(t_0))< \frac{1}{n}$.
	Then, GF converges to zero loss, and converges in direction to a KKT point $\btheta^*$ of Problem~(\ref{eq:optimization problem}), such that the network $\cn_{\btheta^*}$ has at most 
	$32 r + 67$
	linear regions.
\end{theorem}

Since the labels in the dataset $S$ may switch sign $r$ times, then a network that correctly classifies $S$ must contain at least $r$ linear regions. Hence, the theorem implies that GF minimizes the number of linear regions up to a constant factor. We remark that the constants $32$ and $67$ in the above result can be improved, but we preferred here a simpler proof over a tighter bound.



The formal proof of the theorem is given in Appendix~\ref{app:proof of minimize regions}. Below we discuss the high-level approach.
By \thmref{thm:known KKT}, if there exists time $t_0$ such that $\cl(\btheta(t_0)) < \frac{1}{n}$ then GF converges to zero loss, and converges in direction to a KKT point of Problem~(\ref{eq:optimization problem}).
We denote $\cn_\btheta(x) = \sum_{j \in [k]} v_j \sigma(w_j x + b_j)$. 
Assume that $\cn_\btheta$ satisfies the KKT conditions of Problem~(\ref{eq:optimization problem}).
Thus,
there are $\lambda_1,\ldots,\lambda_n \geq 0$ such that for every $j \in [k]$ we have
\begin{equation}
\label{eq:kkt condition w idea}
	w_j 
	= \sum_{i \in [n]} \lambda_i \frac{\partial}{\partial w_j} \left( y_i \cn_{\btheta}(x_i) \right) 
	= \sum_{i \in [n]} \lambda_i y_i v_j \sigma'_{i,j} x_i~,
\end{equation}
where $\sigma'_{i,j}$ is a subgradient of $\sigma$ at $w_j \cdot x_i + b_j$, 
and
$\lambda_i=0$ if $y_i \cn_{\btheta}(x_i) \neq 1$.
Likewise, we have
\begin{equation}
\label{eq:kkt condition b idea}
	b_j 
	= \sum_{i \in [n]} \lambda_i \frac{\partial}{\partial b_j} \left( y_i \cn_{\btheta}(x_i) \right) 
	= \sum_{i \in [n]} \lambda_i y_i v_j \sigma'_{i,j}~.
\end{equation}
In the proof we show using a careful analysis of Eqs.~(\ref{eq:kkt condition w idea})~and~(\ref{eq:kkt condition b idea}) that in an interval $[x_p,x_q]$ where the labels do not switch sign (i.e., $y_p=y_{p+1}=\ldots=y_q$) the network $\cn_\btheta$ has a constant number of kinks. Since the labels switch sign at most $r$ times then we are able to conclude that $\cn_\btheta$ has $\co(r)$ kinks as required. 

\stam{
The formal proof of this claim is a bit technical, but in order to gain some intuition we show below a weaker claim, namely, that $\cn_\btheta$ has at most $O(n)$ kinks.  

Suppose that there are two neurons $j,j' \in [k]$ such that $\sigma'_{i,j}=\sigma'_{i,j'}$ for all $i \in [n]$. Thus, for every input $x_i$ the neuron $j$ is active iff the neuron $j'$ is active. Note the the neuron $j$ has a kink where $w_j x + b_j = 0$, namely at
\[
    x 
    = \frac{-b_j}{w_j}
    = \frac{-\sum_{i \in [n]} \lambda_i y_i v_j \sigma'_{i,j}}{\sum_{i \in [n]} \lambda_i y_i v_j \sigma'_{i,j} x_i}
    = \frac{-\sum_{i \in [n]} \lambda_i y_i \sigma'_{i,j}}{\sum_{i \in [n]} \lambda_i y_i \sigma'_{i,j} x_i}~,
\]
and since $\sigma'_{i,j}=\sigma'_{i,j'}$ for all $i$ then the above equals 
\[
    \frac{-\sum_{i \in [n]} \lambda_i y_i \sigma'_{i,j'}}{\sum_{i \in [n]} \lambda_i y_i \sigma'_{i,j'} x_i}
    = \frac{-b_{j'}}{w_{j'}}~.
\]
Hence, both neurons $j,j'$ have a kink at the same point. 
Note that the network $\cn_\btheta$ may have a kink only where some neuron has a kink. Therefore, the number of kinks in $\cn_\btheta$ is upper bounded by the number of activation patterns of the neurons, which is $O(n)$.
}

\section{Implicit bias leads to a generalization bound}\label{sec:generalization}

Consider a depth-$2$ teacher ReLU network $\cn^*:\reals \to \reals$ of width $r$. 
By \thmref{thm:minimize regions}, if GF reaches a sufficiently small loss at some time $t_0$, it converges in direction to a network $\cn'$ with $\co(r)$ linear regions that classifies the training dataset correctly. Thus, even if the width $k$ of the learned network is extremely large, the result in \thmref{thm:minimize regions} guarantees that GF converges to a network with a small number of linear regions.
Consider the function $f':\reals \to \{-1,1\}$ defined by $f'(x)=\sign(\cn'(x))$. Since $\cn'$ has $\co(r)$ linear regions then $f'$ can be expressed by a polynomial threshold function of degree $\co(r)$. That is, we have $f'(x)=\sign(p'(x))$ where $p'$ is a polynomial of degree $\co(r)$. 
Since $\cn'$ attains $100\%$ classification accuracy on the size-$n$ training dataset, then we can view GF as empirical risk minimization (ERM) w.r.t.\ the $0\text{-}1$ loss over a class of degree-$\co(r)$ polynomial threshold functions in the realizable setting. The VC-dimension of this class is $\co(r)$, and thus we have the following generalization bound w.r.t.\ the $0\text{-}1$ loss (see, e.g., Theorem 6.8 in \citet{shalev2014understanding}). 

\begin{corollary}\label{cor:small_loss_generalizes}
    There exists a universal constant $C_0>0$ such that the following holds. 
    Let $\cn^*:\reals \to \reals$ be a depth-$2$ ReLU network of width $r$.
    Let $\varepsilon,\delta \in (0,1)$ and let
    \[
        n \geq C_0 \cdot \frac{r \log(1/\varepsilon) + \log(1/\delta)}{\varepsilon}~.
    \]
    Let $S$ be a size-$n$ binary classification dataset drawn from a distribution $\cd$ and labeled according to $\cn^*$.
    Consider 
    GF on a depth-$2$ neural network $\cn_\btheta$ w.r.t. the dataset $S$, and
    suppose that there exists time $t_0$ such that $\Lcal(\btheta(t_0))<\frac1n$.
    Then, GF converges 
    in direction to $\btheta^*$ such that 
    the network $\cn_{\btheta^*}$ has at most $32 r + 67$ linear regions, and
    with probability at least $1-\delta$ over the sampling of $S$ we have
    \[
        \pr_{x \sim \cd}\left[\sign(\cn_{\btheta^*}(x))\neq \sign(\cn^*(x)) \right] \leq \varepsilon~.
    \]
\end{corollary}

\subsection*{Acknowledgements}

Work done while GV was at the Weizmann Institute of Science.
We thank Noam Razin and Gilad Yahudai for pointing out several relevant papers to discuss in the related work section.

\bibliographystyle{abbrvnat}
\bibliography{bib}

\appendix

\section{Preliminaries on the Clarke subdifferential and the KKT conditions}
\label{app:KKT}

Below we define the Clarke subdifferential, and review the definition of the KKT conditions for non-smooth optimization problems (cf.\ \cite{lyu2019gradient,dutta2013approximate}).

Let $f: \reals^d \to \reals$ be a locally Lipschitz function. The Clarke subdifferential \citep{clarke2008nonsmooth} at $\bx \in \reals^d$ is the convex set
\[
	\partial^\circ f(\bx) := \text{conv} \left\{ \lim_{i \to \infty} \nabla f(\bx_i) \; \middle| \; \lim_{i \to \infty} \bx_i = \bx,\; f \text{ is differentiable at } \bx_i  \right\}~.
\]
If $f$ is continuously differentiable at $\bx$ then $\partial^\circ f(\bx) = \{\nabla f(\bx) \}$.
For the Clarke subdifferential the chain rule holds as an inclusion rather than an equation.
That is, for locally Lipschitz functions $z_1,\ldots,z_n:\reals^d \to \reals$ and $f:\reals^n \to \reals$, we have
\[
	\partial^\circ(f \circ \bz)(\bx) \subseteq \text{conv}\left\{ \sum_{i=1}^n \alpha_i \bh_i: \balpha \in \partial^\circ f(z_1(\bx),\ldots,z_n(\bx)), \bh_i \in \partial^\circ z_i(\bx) \right\}~.
\]

Consider the following optimization problem
\begin{equation}
\label{eq:KKT nonsmooth def}
	\min f(\bx) \;\;\;\; \text{s.t. } \;\;\; \forall n \in [N] \;\; g_n(\bx) \leq 0~,
\end{equation}
where $f,g_1,\ldots,g_n : \reals^d \to \reals$ are locally Lipschitz functions. We say that $\bx \in \reals^d$ is a \emph{feasible point} of Problem~(\ref{eq:KKT nonsmooth def}) if $\bx$ satisfies $g_n(\bx) \leq 0$ for all $n \in [N]$. We say that a feasible point $\bx$ is a \emph{KKT point} if there exists $\lambda_1,\ldots,\lambda_N \geq 0$ such that 
\begin{enumerate}
	\item $\zero \in \partial^\circ f(\bx) + \sum_{n \in [N]} \lambda_n \partial^\circ g_n(\bx)$;
	\item For all $n \in [N]$ we have $\lambda_n g_n(\bx) = 0$.
\end{enumerate}

\section{Details on Example~\ref{ex:not global}} \label{app:example details}

We use here the notation $\Phi(\btheta; x) = \cn_\btheta(x)$, and denote by $\sigma'(\cdot)$ a sub-gradient of $\sigma$, namely, $\sigma(z)=\onefunc[z > 0]$ if $z \neq 0$ and $\sigma'(0) \in [0,1]$ (the exact value in this case is not important here).
For every $j \in \{1,2\}$ we have 
\begin{align*}
	\nabla_{w_j} \cl(\btheta) 
	&= \frac{1}{2} \sum_{i=1}^2 \ell'(y_i \Phi(\btheta; x_i)) \cdot y_i \nabla_{w_j} \Phi(\btheta; x_i)
	\\
	&= \frac{1}{2} \sum_{i=1}^2 \ell'(v_1 \sigma(w_1 x_i + b_1) + v_2 \sigma(w_2 x_i + b_2)) \cdot v_j \sigma'(w_j x_i + b_j) x_i~.
\end{align*}
Likewise,
\begin{align*}
	\nabla_{v_j} \cl(\btheta) 
	&= \frac{1}{2} \sum_{i=1}^2 \ell'(y_i \Phi(\btheta; x_i)) \cdot y_i \nabla_{v_j} \Phi(\btheta; x_i)
	\\
	&= \frac{1}{2} \sum_{i=1}^2 \ell'(v_1 \sigma(w_1 x_i + b_1) + v_2 \sigma(w_2 x_i + b_2)) \cdot \sigma(w_j x_i + b_j)
\end{align*}
and
\begin{align*}
	\nabla_{b_j} \cl(\btheta) 
	&= \frac{1}{2} \sum_{i=1}^2 \ell'(y_i \Phi(\btheta; x_i)) \cdot y_i \nabla_{b_j} \Phi(\btheta; x_i)
	\\
	&= \frac{1}{2} \sum_{i=1}^2 \ell'(v_1 \sigma(w_1 x_i + b_1) + v_2 \sigma(w_2 x_i + b_2)) \cdot v_j \sigma'(w_j x_i + b_j)~.
\end{align*}

Note that if $w_2=b_2=v_2=0$ then we have $\nabla_{w_2} \cl(\btheta) = \nabla_{v_2} \cl(\btheta) = \nabla_{b_2} \cl(\btheta) = 0$. Since these parameters are initialized at zero, then they remain zero throughout the training.
Moreover, Suppose that $w_1=0$ and $b_1 = v_1 = \alpha$ for some $\alpha>0$, and that $w_2=b_2=v_2=0$, then we have
\[
	- \frac{d w_1}{dt} =
	\nabla_{w_1} \cl(\btheta) 
	= \frac{1}{2} \sum_{i=1}^2 \ell'(\alpha \sigma(\alpha)) \cdot \alpha \sigma'(\alpha) x_i
	= \frac{1}{2} \left( \ell'(\alpha^2) \cdot \alpha \cdot 4 + \ell'(\alpha^2) \cdot \alpha \cdot (-4) \right)
	=0~,
\]
\[
	- \frac{d v_1}{dt} =
	\nabla_{v_1} \cl(\btheta) 
	= \frac{1}{2} \sum_{i=1}^2 \ell'(\alpha \sigma(\alpha)) \cdot \sigma(\alpha)
	= \frac{1}{2} \sum_{i=1}^2 \ell'(\alpha^2) \cdot \alpha~,
\]
\[
	- \frac{d b_1}{dt} =
	\nabla_{b_1} \cl(\btheta) 
	= \frac{1}{2} \sum_{i=1}^2 \ell'(\alpha \sigma(\alpha)) \cdot \alpha \sigma'(\alpha)
	= \frac{1}{2} \sum_{i=1}^2 \ell'(\alpha^2) \cdot \alpha~.
\]
Hence, for every $t$ we have $w_1(t)=0$ and $b_1(t) = v_1(t) = \alpha(t)$ where $\alpha(t)>0$ is monotonically increasing. 

As a result, the KKT point $\btheta^*$ is such that $w^*_2=b^*_2=v^*_2=w^*_1=0$, and $b^*_1=v^*_1 = \alpha^*$ for some $\alpha^*>0$. Since $\btheta^*$ satisfies the KKT conditions of Problem~(\ref{eq:optimization problem}), then we have
\[
	\alpha^* = b^*_1 = \sum_{i=1}^2 \lambda_i y_i \nabla_{b_1} \Phi(\btheta^*; x_i)~,	
\]
where $\lambda_i \geq 0$ and $\lambda_i=0$ if $y_i \Phi(\btheta^*; x_i) \neq 1$. Hence, there is $i$ such that $y_i \Phi(\btheta^*; x_i) = 1$. Thus, we have $1= y_i \Phi(\btheta^*; x_i) = (\alpha^*)^2$ which implies $\alpha^*=1$. Therefore $\btheta^*=\btheta(0)$.

\section{Proof of \thmref{thm:optimization}}\label{app:optimization_proof}

Before we prove the theorem, we first state a few definitions that are specific for this appendix. Let $S = \{(x_i,y_i)\}_{i=1}^n \subseteq [-R,R] \times \{-1,1\}$ be a dataset such that $x_1 < \ldots < x_n$ and let $I = \{i \in [n-1] : y_i \neq y_{i+1} \}$ where we denote the elements of $I$ using $i_1<\ldots<i_r$, and $i_0=-R,i_{r+1}=R$, where $r=|I|$. For all $j\in[r+1]$, define $I_j=(i_{j-1},i_j+1)$ which is the $j$-th interval where the instances in the data do not change their classification.\footnote{We note that these intervals overlap and thus do contain instances that change classification with respect to the teacher network, but not with respect to the sample.} Given some function $\Lcal$ and real number $\alpha$, we let $L_{\alpha}^{+}(\Lcal)\coloneqq\set{\btheta:\Lcal(\btheta)\ge \alpha}$ denote the $\alpha$-superlevel set of $\Lcal$.

Next, we state the following definitions, which establish sufficient conditions for our objective function to be well-behaved in the sense of having a strict direction of descent in a certain neighborhood.

\begin{definition}[Separability]\label{def:separability}
    Under \asmref{asm:data}, we say that $\btheta$ is separable from $S$ with positive constants $\gamma,m<M,q<Q$ if for all $j\in[r+1]$, there exist three neurons with weights and biases denoted by $w_i(I_j)$ and $b_i(I_j)$ for $i\in[3]$, and breakpoints $\beta_{1}(I_j)<\beta_{2}(I_j)<\beta_{3}(I_j)\in I_j$, which satisfy the following items:
    \begin{enumerate}
        \item\label{item:1}
        $m\le|w_{i}(I_j)|$ and $|b_{i}(I_j)|\le M$ for all $i\in[3]$.
        \item\label{item:2}
        There exist four neurons, two with breakpoints in each of the intervals $(-\infty,0),(0,\infty)$, that are distinct from the neurons in the previous item, are active on all the data instances and whose weights $w'_i,b'_i$ satisfy $m\le|w'_i|$ and $|b'_i|\le M$ for all $i\in[4]$. Moreover, their breakpoints satisfy $|\beta_i|\le Q$ for all $i\in[4]$.
        \item\label{item:3}
        $q\le\beta_{2}(I_j)-\beta_{1}(I_j),\beta_{3}(I_j)-\beta_{2}(I_j)\le Q$ for all $j\in[r+1]$.
        \item\label{item:4}
        $|\beta_{1}(I_{j+1}) - \beta_{3}(I_j)| \le Q$ for all $j\in[r]$.
        \item\label{item:5}
        $x_{i_j+1}-x_{i_j}\ge\gamma$ for all $j\in[r]$.
    \end{enumerate}
    If the triplet of neurons satisfying the above items in an interval is not distinct, we assume w.l.o.g.\ that $\beta_1(\cdot)$ and $\beta_3(\cdot)$ return the left-most and right-most breakpoints satisfying the above, respectively.
\end{definition}



The following definition is used to describe a neighborhood around the initialization point in which our separability assumption above holds.

\begin{definition}[$\Delta$-hidden Neighborhood]\label{def:delta_hidden_neighborhood}
    Given a network $\Ncal(\cdot)$, weights $\btheta$ and a constant $\Delta\ge0$, we define the $\Delta$-hidden Neighborhood of $\Ncal$ at $\btheta$ as the set
    \[
        U_{\Delta}(\btheta) \coloneqq \set{\btheta'=[\bw',\bb',\bv']:\norm{(w_j,b_j)-(w'_j,b'_j)}_2\le\Delta \hskip 0.3cm\forall j\in[k],\hskip 0.3cm \bv'\in\reals^k}.
    \]
    
\end{definition}

That is, the neighborhood of balls of radius $\Delta$ centered at each hidden neuron of $\btheta$ and where the output neuron weights are arbitrary.

Following the above definitions, the following auxiliary lemmas will be used in the proof of the theorem. The technical lemma below establishes that a certain binary matrix is invertible and provides a bound on the spectral norm of its inverse.

\begin{lemma}\label{lem:masking_matrix_inv}
    Suppose that $A\in\{0,1\}^{d\times d}$, such that the first row of $A$ is all-ones, and each subsequent row $i$ is either $(1,\ldots,1,0,\ldots,0)$ with $i-1$ leading ones or $(0,\ldots,0,1,\ldots,1)$ with $i-1$ leading zeros. Then $A$ is invertible and we have $\spnorm{A^{-1}}\le d$.
\end{lemma}

\begin{proof}
    The invertability of $A$ follows from the fact that the first row of $A$ is an all-ones vector, since we can use elementary row operations to change all subsequent rows to start with a `0' and end with a `1' if needed, resulting in an upper triangular matrix with all-ones on its main diagonal which is thus invertible. To bound $\spnorm{A^{-1}}$, let $I_d\in\{0,1\}^{d\times d}$ denote the identity matrix. We will use Gaussian elimination to compute the entries of $A^{-1}$. We first subtract the first row from all the other rows that do not have leading zeros and then multiply by the constant $-1$. Performing the same operation on $I_d$ results in a matrix $B$ whose rows are either standard unit vectors or the vector $(1,\ldots,1,0,1,\ldots,1)$. The resulting matrix after performing these operations on $A$ is and upper triangular matrix with ones in all of its diagonal and above the diagonal entries. Since it is readily seen that the inverse of such a matrix is a matrix with all zero entries except for the main diagonal which is all-ones and the first diagonal above it which comprises of all $-1$'s. Denote this matrix using $B'$, we have that the inverse of $A$ is given by $B\cdot B'$. The entries of $A^{-1}$ therefore must consist of dot products of a standard unit vector and vectors $(0,\ldots,0,1,-1,0,\ldots,0)$, or the vector $(1,\ldots,1,0,1,\ldots,1)$ and vectors $(0,\ldots,0,1,-1,0,\ldots,0)$. In both cases the dot product is an element of $\{-1,0,1\}$, and therefore we can bound $\spnorm{A^{-1}}$ by the Frobenius norm of $A^{-1}$ which is at most $d$.
\end{proof}

The following key lemma establishes that when $\btheta$ is separable from $S$ then there exists a direction in weight space which strictly decreases our objective value.

\begin{lemma}\label{lem:grad_norm_lbound}
    Under \asmref{asm:data}, suppose that $\btheta$ is separable from $S$ with constants $\gamma,m,M,q,Q$, and that $\Lcal(\btheta)\ge\frac{1}{2n}$. Then
    \[
        \frac12\norm{\nabla\Lcal(\btheta)}_2^2 \ge \frac{\gamma^2q^2m^6}{259200n^4Q^2M^4}.
    \]
\end{lemma}

\begin{proof}
    Since $\Lcal(\btheta) > \frac{1}{2n}$, there must exist some $i\in[n]$ such that $\ell(y_i\Phi(\btheta;x_i)) > \frac{1}{2n}$. We now consider two possible cases, depending on the location of $x_i$ with respect to the breakpoints whose existence is guaranteed by \defref{def:separability}, where will show the existence of a direction $\bu$ which guarantees that $\Lcal(\btheta)$ is strictly decreasing.
    \begin{itemize}
        \item
        Suppose that $x_i\in I_\ell$ for some $\ell\in[r+1]$, such that $x_i\in(\beta_1(I_\ell), \beta_3(I_\ell))$. Assume without loss of generality that $x_i\in(\beta_2(I_\ell), \beta_3(I_\ell))$ (the proof is symmetric otherwise), and for ease of notation denote $\beta_2\coloneqq\beta_1(I_\ell),\beta_3\coloneqq\beta_2(I_\ell),\beta_4\coloneqq\beta_3(I_\ell)$. Then by \itemref{item:2} in the separability assumption, there exists a breakpoints $\beta_1<\beta_2$ such that all data instances in $(\beta_2,\beta_4)$ have the same classification and the neuron with breakpoint at $\beta_1$ is active on all the data instances. Moreover, \itemref{item:2} also guarantees the existence of two breakpoints that are distinct from the previous ones, which we denote by $\beta_5,\beta_6$, where $\beta_5<\beta_1<0$ and $\beta_6>0$ are active on all the data points. We will now show the existence of a depth-2 ReLU network which consists of six hidden neurons with weights $\bw=(w_1,\ldots,w_6)$ and biases $\bb=(b_1,\ldots,b_6)$ (corresponding to the breakpoints $\beta_1,\ldots,\beta_6$ defined above) which computes the piece-wise linear function $f:[\beta_1,\beta_6)\to\reals$ given by
        \[
            f(x)\coloneqq\begin{cases}
                 0 & x\in[\beta_1,\beta_2]\\
                 \frac{1}{\beta_3-\beta_2}x - \frac{\beta_2}{\beta_3-\beta_2} & x\in(\beta_2,\beta_3)\\
                 \frac{1}{\beta_3-\beta_4}x-\frac{\beta_4}{\beta_3-\beta_4} & x\in[\beta_3,\beta_4]\\
                 0 & x\in(\beta_4,\beta_6)
            \end{cases}.
        \]
        The intuition behind the approximation is that we can use the first four neurons to approximate the slopes of the function $f$, and the last two remaining neurons to simulate a bias term which would shift the function approximated by the network to overlap $f$ in the relevant domain of approximation (see \figref{fig:f1} for an illustration). 
        
        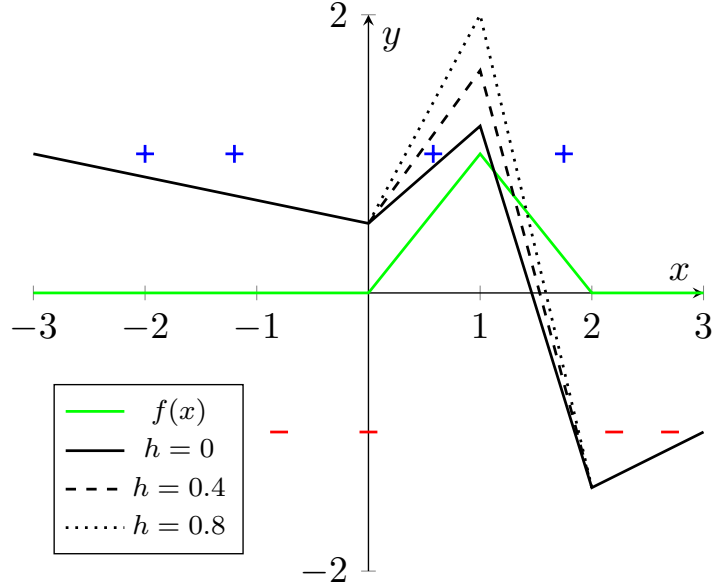
\begin{figure}[t!]
            \centering
            \scalebox{1.3}{
            \begin{tikzpicture}
                \begin{axis}[legend style={font=\scriptsize}, legend pos=south west,
                    axis x line=middle, axis y line=middle,
                    ymin=-2, ymax=2, ytick={-2,2}, ylabel=$y$,
                    xmin=-3, xmax=3, xtick={-3,...,3}, xlabel=$x$,
                    domain=-3:3,samples=101,
                    ]
                    \addplot[green,thick] coordinates {(-3,0)(0,0)(1,1)(2,0)(3,0)};
                    \addplot[black,thick] coordinates {(-3,1)(0,0.5)(1,1.2)(2,-1.4)(3,-1)};
                    \addplot[black,dashed,thick] coordinates {(0,0.5)(1,1.6)(2,-1.4)};
                    \addplot[black,dotted,thick] coordinates {(0,0.5)(1,2)(2,-1.4)};
                    \addplot[only marks,blue,thick,mark=+,mark options={scale=1.3}] coordinates {(-2,1)(-1.2,1)(0.58,1)(1.75,1)};
                    \addplot[only marks,red,thick,mark=-,mark options={scale=1.3}] coordinates {(-0.8,-1)(0,-1)(2.2,-1)(2.7,-1)};
                    \legend{$f(x)$,$h=0$,$h=0.4$,$h=0.8$}
                \end{axis}
            \end{tikzpicture}
            }
            \caption{The plots of $f(x)$ (green) and $\Phi(\btheta;x)+h\cdot f(x)$ (black) for various values of $h$. Moving $\btheta$ in the direction of $\bu$ which computes $f(\cdot)$ strictly decreases the loss over the positively-labeled instances in the interval $(0,2)$ without affecting the rest of the dataset. Best viewed in color.}
            \label{fig:f1}
        \end{figure}

        More formally, define $W\coloneqq\diag(w_1,\ldots,w_4)\in\reals^{4\times4}$, define the masking matrix $A\in\reals^{4\times4}$ with entries $a_{j,j'}=\one{j\text{-th neuron is active on the interval starting with }\beta_{j'}}$ and let 
        \[
            \bd=\p{0,\frac{1}{\beta_3-\beta_2},\frac{1}{\beta_3-\beta_4},0}.
        \]
        Thus, to match the slopes computed by the a depth-2 ReLU network with weights $\bw,\bb$ to those of $f$, we first want the output neuron's weights $\bv=[\bv_1,\bv_2]\in\reals^4\times\reals^2$ to satisfy the equality $A^{\top}W\bv_1=\bd$. To this end, we have by \lemref{lem:masking_matrix_inv} that $A$ is invertible and $\spnorm{A^{-\top}}\le4$. It then follows that $\bv_1 = W^{-1}A^{-\top}\bd$, which entails
        \begin{equation}\label{eq:first_v_1_bound}
            \norm{\bv_1}_2 \le \spnorm{W^{-1}}\spnorm{A^{-\top}}\norm{\bd}_2.  \le \frac{8\sqrt{2}}{qm},
        \end{equation}
        where we used the separability assumption, implying that $\norm{\bd}_2\le\sqrt{2}q^{-1}$ due to \itemref{item:3} which guarantees that $\beta_3-\beta_2,\beta_4-\beta_3\ge q$, and the lower bound assumption $|w_i|\ge m$ for all $i\in[4]$ which holds by Items~\ref{item:1}~and~\ref{item:2}. Next, we use the two neurons with breakpoints at $\beta_5,\beta_6$ to shift the network by a constant so that it overlaps with $f$ on the interval $[-R,R]$. To perform this shift, we first compute the magnitude by which we wish to shift which is given by the expression
        \[
            b_0\coloneqq-\sum_{j=1}^4v_jb_j\one{w_jx+b_j>0\hskip 0.2cm \forall x\in(\beta_{j'},\beta_{j'+1})}.
        \]
        Letting
        \[
            P\coloneqq\p{\begin{matrix}
                    b_5 & b_6\\
                    w_5 & w_6
            \end{matrix}},
        \]
        we have that the neurons with breakpoints at 
        $\beta_5,\beta_6$ compute a function which equals $b_0$ on the interval $[\beta_5,\beta_6]$ when the equality $P\cdot\bv_2=(b_0,0)^{\top}$ is satisfied. We will now compute $P^{-1}$ and show that it is well-defined. The inverse of a $2\times2$ matrix is given by
        \[
            P^{-1}=\frac{1}{b_5w_6-b_6w_5}\p{\begin{matrix}
                    w_6 & -b_6\\
                    -w_5 & b_5
            \end{matrix}}
            =\frac{1}{\beta_6-\beta_5}\p{\begin{matrix}
                    \frac{1}{w_5} & -\frac{b_6}{w_5w_6}\\
                    -\frac{1}{w_6} & \frac{b_5}{w_5w_6}.
            \end{matrix}}
        \]
        Using the above, we can upper bound the spectral norm of $P^{-1}$ by upper bounding $1/(\beta_6-\beta_5)$ with $\frac{1}{2R}\le\frac12$ since $\beta_5,\beta_6$ are outside the interval $[-R,R]$ at opposite sides and $R\ge1$ by \asmref{asm:data}, and by upper bounding the spectral norm of the matrix with its Frobenius norm by using \itemref{item:2}, to obtain
        \[
            \spnorm{P^{-1}} \le \frac{M}{m^2}.
        \]
        Similarly, we derive an upper bound on $|b_0|$ using Cauchy-Schwartz and Items~\ref{item:1}~and~\ref{item:2} to obtain
        \[
            |b_0|\le \norm{\bv_1}_2\norm{\bb}_2\le2M\norm{\bv_1}_2.
        \]
        With the above, we can bound the norm of $\bv_2$ as follows
        \begin{equation}\label{eq:linear_shift}
            \norm{\bv_2}_2=\spnorm{P^{-1}\cdot(b_0,0)^{\top}} \le \frac{2M^2}{m^2}\norm{\bv_1}_2.
        \end{equation}
        We now define $\bu$ as the all-zero vector, except for the six output neuron entries corresponding to the neurons with breakpoints $\beta_1,\ldots,\beta_6$, where the coordinates of $\bu$ take the values $y_iv_1,\ldots,y_iv_6$. Note that this entails
        \begin{equation}\label{eq:first_u_bound}
            \norm{\bu}_2=\sqrt{\norm{\bv_1}_2^2+\norm{\bv_2}_2^2}\le \norm{\bv_1}_2\sqrt{1+\frac{4M^4}{m^4}} \le\frac{8}{qm}\sqrt{2+\frac{8M^4}{m^4}} \le 8\sqrt{10}\frac{M^2}{qm^3},
        \end{equation}
        where we have used \eqref{eq:first_v_1_bound} and the fact that $1<M/m$. Next, we have for all $j\in[n]$ and $h>0$ that
        \[
            y_{j}\Phi\p{\btheta + \frac{h}{\norm{\bu}}\bu;x_{j}} = y_{j}\p{\Phi(\btheta;x_{j}) + \frac{h}{\norm{\bu}}y_if(x_j)}.
        \]
        Observe that all data points satisfy $x_j\in[\beta_5,\beta_6]$, and are therefore unaffected by the value $\Phi(\cdot,x)$ attains for $x$'s outside of this interval. Additionally, $f(x)=0$ for all $x\in[\beta_1,\beta_2]\cup(\beta_4,\infty)$, which also keeps $\Phi(\cdot,x)$ unaffected by moving in the direction of $\bu$. Moreover, since the sign of data instances in $x_j\in(\beta_2,\beta_4]$ is always $y_i$, we have that
        \[
            y_{j}\Phi\p{\btheta + \frac{h}{\norm{\bu}}\bu;x_{j}} = y_{j}\Phi(\btheta;x_{j}) + \frac{h}{\norm{\bu}}f(x_j) \ge y_{j}\Phi(\btheta;x_{j}).
        \]
        Lastly, for $x_i$ we have that $x_i\in[\beta_3,\beta_4]$, and that $x_i$ is at distance at least $\gamma$ from the boundary by \itemref{item:5} in our separability assumption. By \itemref{item:3}, this implies that $f(x_i)$ is at least $\frac{\gamma}{\beta_4-\beta_3}\ge\frac{\gamma}{Q}$, which with the above equation and our bound from \eqref{eq:first_u_bound} yields
        \begin{equation}\label{eq:first_case_bound}
            y_{i}\Phi\p{\btheta + \frac{h}{\norm{\bu}}\bu;x_{i}} \ge y_{i}\Phi(\btheta;x_{i}) +h\frac{\gamma}{Q\norm{\bu}} \ge y_{i}\Phi(\btheta;x_{i}) + h\frac{\gamma qm^3}{8\sqrt{10}QM^2}.
        \end{equation}
        
        \item
        Suppose that $x_i\in I_\ell$ for some $\ell\in\{1,\ldots,r+1\}$, such that $x_i\notin(\beta_1(I_\ell), \beta_3(I_\ell))$. Assume without loss of generality that $x_i\le\beta_4\coloneqq\beta_1(I_\ell)$ (the proof is symmetric otherwise). Then by our separability assumption, there exist $\beta_2\coloneqq\beta_1(I_{\ell-1}),\beta_3\coloneqq\beta_3(I_{\ell-1}),\beta_5\coloneqq\beta_3(I_{\ell})$ and $\beta_1$ whose neuron is active on all the data points where $\beta_1\le\beta_2$. Note that by the definition of $\beta_i(\cdot)$, it must hold that $\beta_4$ is the smallest element in the interval $I_{\ell}$ which satisfies Items~\ref{item:1}~and~\ref{item:3} (since otherwise we would have that $x_i\in(\beta_3,\beta_5)\subseteq I_\ell$, which is handled in the previous case), and therefore $\beta_3\notin I_{\ell}$, implying that $x_i\in(\beta_3,\beta_4)$. Moreover, we note that we may assume that $\ell>1$, since otherwise we can take the smallest two breakpoints that are active on all the data along with $\beta_4$ which reduces us to the previous case. We thus denote the largest data instance in $I_{\ell-1}$ as $x_{i-1}$ which implies $x_{i-1}< x_i$. Lastly, \itemref{item:2} also guarantees the existence of two breakpoints that are distinct from the previous ones, which we denote by $\beta_6,\beta_7$, where $\beta_6<\beta_1<0$ and $\beta_7>0$ are active on all the data points.
        
        Following a similar approach as in the previous case, we will now show the existence of a depth-2 ReLU network which consists of seven hidden neurons with weights $\bw=(w_1,\ldots,w_7)$ and biases $\bb=(b_1,\ldots,b_7)$ and computes the piece-wise linear function $f:[\beta_1,\beta_7)\to\reals$ given by
        \[
            f(x)\coloneqq\begin{cases}
                0 & x\in[\beta_1,\beta_2]\\
                \frac{1}{\beta_2-\beta_3}x - \frac{\beta_2}{\beta_2-\beta_3} & x\in(\beta_2,\beta_3)\\
                \frac{1}{x_{i-1}-\beta_3}x - \frac{x_{i-1}}{x_{i-1}-\beta_3} & x\in[\beta_3,\beta_4]\\
                \frac{\beta_4-x_{i-1}}{x_{i-1}-\beta_3}\p{\frac{1}{\beta_4-\beta_5}x-\frac{\beta_5}{\beta_4-\beta_5}} & x\in(\beta_4,\beta_5)\\
                 0 & x\in[\beta_5,\beta_7)
            \end{cases},
        \]
        where the first five neurons are used to compute the slopes of $f$ and the remaining last two neurons are used to simulate a bias term to shift the network to accord with $f$ in its domain (see \figref{fig:f2} for an illustration). 
        
        \begin{figure}[t!]
            \centering
            \scalebox{1.3}{
            \begin{tikzpicture}
                \begin{axis}[legend style={font=\scriptsize}, legend pos=south west,
                    axis x line=middle, axis y line=middle,
                    ymin=-2, ymax=2, ytick={-2,2}, ylabel=$y$,
                    xmin=-3, xmax=3, xtick={-3,...,3}, xlabel=$x$,
                    domain=-3:3,samples=101,
                    ]
                    \addplot[green,thick] coordinates {(-3,0)(-1,0)(-0.5,-0.5)(1,1)(2,0)(3,0)};
                    \addplot[black,thick] coordinates {(-3,1)(-1,0.2)(-0.5,0.8)(1,1.2)(2,-1.4)(3,-1)};
                    \addplot[black,dashed,thick] coordinates {(-3,1)(-1,0.2)(-0.5,0.6)(1,1.6)(2,-1.4)(3,-1)};
                    \addplot[black,dotted,thick] coordinates {(-3,1)(-1,0.2)(-0.5,0.4)(1,2)(2,-1.4)(3,-1)};
                    \addplot[only marks,blue,thick,mark=+,mark options={scale=1.3}] coordinates {(-2,1)(-1.2,1)(0.6,1)(1.75,1)};
                    \addplot[only marks,red,thick,mark=-,mark options={scale=1.3}] coordinates {(-0.8,-1)(0,-1)(2.2,-1)(2.7,-1)};
                    \legend{$f(x)$,$h=0$,$h=0.4$,$h=0.8$}
                \end{axis}
            \end{tikzpicture}
            }
            \caption{The plots of $f(x)$ (green) and $\Phi(\btheta;x)+h\cdot f(x)$ (black) for various values of $h$. Moving $\btheta$ in the direction of $\bu$ which computes $f(\cdot)$ strictly decreases the loss over the positively-labeled instances in the interval $(0,2)$ without degrading the loss over the rest of the dataset. Unlike the previous simpler case, since $\btheta$ has no breakpoint between the negative instance at $x=0$ and the positive instance at $x=0.6$, we use a function $f(x)$ which pivots around $x=0$ to prevent the prediction over the negatively-labeled instances in the interval $[-1,0]$ from increasing.
            Best viewed in color.}
            \label{fig:f2}
        \end{figure}
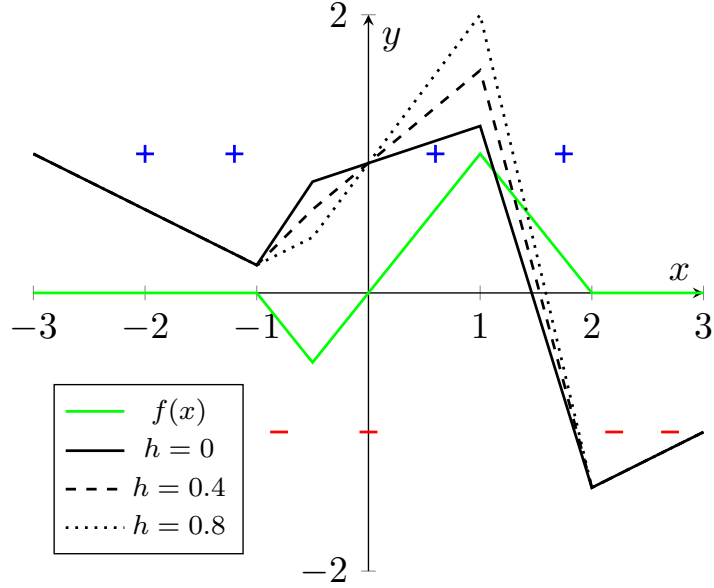
        
        Define $W\coloneqq\diag(w_1,\ldots,w_5)\in\reals^{5\times5}$, define the masking matrix $A\in\reals^{5\times5}$ with entries $a_{j,j'}=\one{j\text{-th neuron is active on the interval starting with }\beta_{j'}}$ and let
        \[
            \bd=\p{0,\frac{1}{\beta_2-\beta_3},\frac{1}{x_{i-1}-\beta_3},\frac{\beta_4-x_{i-1}}{x_{i-1}-\beta_3}\cdot\frac{1}{\beta_4-\beta_5},0}.
        \]
        
        Thus, to match the slopes computed by the a depth-2 ReLU network with weights $\bw,\bb$ to those of $f$, we first want the output neuron's weights $\bv=[\bv_1,\bv_2]\in\reals^5\times\reals^2$ to satisfy the equality $A^{\top}W\bv_1=\bd$. To this end, we have by \lemref{lem:masking_matrix_inv} that $A$ is invertible and $\spnorm{A^{-\top}}\le5$. It then follows that $\bv_1 = W^{-1}A^{-\top}\bd$, which entails
        \[
            \norm{\bv_1}_2 \le \spnorm{W^{-1}}\spnorm{A^{-\top}}\norm{\bd}_2 \le 5\spnorm{W^{-1}}\norm{\bd}_2.
        \]

        To upper bound the above, we first bound $\norm{\bd}_2$. By our separability assumption we have
        \begin{align*}
            \norm{\bd}_2 &\le \sqrt{\frac{1}{q^2} + \frac{1}{(x_{i-1}-\beta_3)^2} + \frac{4Q^2}{q^2(x_{i-1}-\beta_3)^2}}\\
            &=
            \sqrt{ \frac{q^2+(x_{i-1}-\beta_3)^2+4Q^2}{q^2(x_{i-1}-\beta_3)^2}}
            \le \frac{1}{q(x_{i-1}-\beta_3)}\sqrt{8Q^2+q^2}
        \end{align*}
        where we used \itemref{item:3} to upper bound the denominators, and \itemref{item:4} which entails $\beta_4-x_{i-1},x_{i-1}-\beta_3\le\beta_4-\beta_3\le Q$. By \itemref{item:1}, we have $|w_i|\ge m$ for all $i\in[5]$, and thus we obtain
        \begin{equation}\label{eq:second_v_1_bound}
            \norm{\bv_1}_2\le \frac{5}{mq(x_{i-1}-\beta_3)}\sqrt{40Q^2+5q^2} \le \frac{15\sqrt{5}Q}{mq(x_{i-1}-\beta_3)}.
        \end{equation}
        Now, similarly to the previous case, we wish to shift the function computed by the depth-2 ReLU network by a constant. This is done in the exact same manner as in the previous case, where we shift it by a magnitude given by
        \[
            b_0=-\sum_{j=1}^5v_jb_j\one{w_jx+b_j>0\hskip 0.2cm \forall x\in(\beta_{j'},\beta_{j'+1})}.
        \]
        Bounding $|b_0|$ using its above definition, Cauchy-Schwartz and \itemref{item:1} in our separability assumption, we obtain
        \[
            |b_0|\le \norm{\bv_1}_2\norm{\bb}_2 \le \sqrt{5}M\norm{\bv_1}_2.
        \]
        From the above and \eqref{eq:linear_shift}, we can shift the network by the desired magnitude using a vector $\bv_2=(v_6,v_7)$ satisfying
        \[
            \norm{\bv_2}_2\le \frac{\sqrt{5}M^2}{m^2}\norm{\bv_1}_2.
        \]
        We now define $\bu$ as the all-zero vector, except for the output neuron entries corresponding to the neurons with breakpoints $\beta_1,\ldots,\beta_7$, where the coordinates of $\bu$ take the values $y_iv_1,\ldots,y_iv_7$. Note that this entails
        \begin{equation}\label{eq:second_u_bound}
            \norm{\bu}_2 = \sqrt{\norm{\bv_1}_2^2+\norm{\bv_2}_2^2}\le
            \norm{\bv_1}_2\sqrt{1+\frac{5M^4}{m^4}} \le
            \frac{15\sqrt{30}QM^2}{m^3q(x_{i-1}-\beta_3)} \le \frac{90QM^2}{m^3q(x_{i-1}-\beta_3)},
        \end{equation}
        where we used \eqref{eq:second_v_1_bound} and the fact that $1<M/m$. We therefore have for all $j\in[n]$ and $h>0$ that
        \[
            y_{j}\Phi\p{\btheta + \frac{h}{\norm{\bu}}\bu;x_{j}} = y_{j}(\Phi(\btheta;x_{j}) + \frac{h}{\norm{\bu}}y_if(x_j)).
        \]
        Observe that all data points satisfy $x_j\in[\beta_6,\beta_7]$, and are therefore unaffected by the value $\Phi(\cdot,x)$ attains for $x$'s outside of this interval. Additionally, $f(x)=0$ for all $x\in[\beta_1,\beta_2]\cup[\beta_5,\infty)$, which also keeps $\Phi(\cdot,x)$ unaffected by moving in the direction of $\bu$. In the interval $(\beta_2,x_{i-1})$, the sign of $\Phi(\cdot,x)$ is $-y_i$, and in the interval $(x_{i-1},\beta_5)$ its sign changes to $y_i$. For this reason, similarly to the previous case, we have that

        \begin{equation}\label{eq:f_non_decreasing}
            y_{j}\Phi\p{\btheta + \frac{h}{\norm{\bu}}\bu;x_{j}} = y_{j}\Phi(\btheta;x_{j}) + \frac{h}{\norm{\bu}}f(x_j) \ge y_{j}\Phi(\btheta;x_{j}),
        \end{equation}
        for all $x_j\in(\beta_2,\beta_5)$.
        Lastly, for $x_i$ we have that $x_i\in[\beta_3,\beta_4]$, and that $x_i$ is at distance at least $\gamma$ from $x_{i-1}$ by \itemref{item:5} in our separability assumption. This implies that
        \[
            f(x_i) = \frac{x_i}{x_{i-1}-\beta_3} - \frac{x_{i-1}}{x_{i-1}-\beta_3} \ge \frac{\gamma}{x_{i-1}-\beta_3},
        \]
        and therefore
        $hf(x_i)\ge\frac{\gamma}{x_{i-1}-\beta_3}h$, which with Eqs.~(\ref{eq:second_u_bound},\ref{eq:f_non_decreasing}) implies that
        \begin{equation}\label{eq:x_i_strictly_increasing}
            y_{i}\Phi\p{\btheta + \frac{h}{\norm{\bu}}\bu;x_{i}} \ge y_{i}\Phi(\btheta;x_{i}) +h\frac{\gamma}{\norm{\bu}(x_{i-1}-\beta_3)} \ge y_{i}\Phi(\btheta;x_{i}) + h\frac{\gamma qm^3}{90QM^2},
        \end{equation}
        which is a weaker lower bound than the one derived in \eqref{eq:first_case_bound}, and thus always holds if \eqref{eq:first_case_bound} is satisfied.
    \end{itemize}

    We now turn to lower bound the norm of the gradient by analyzing the directional derivative of $\Lcal(\cdot)$ in the direction of the vector $\bu$ defined by the above two cases. To this end, denote $\bar{\bu}=\frac{\bu}{\norm{\bu}}$ and compute
    \begin{align}
        \norm{\nabla\Lcal(\btheta)}_2 &\ge \abs{\inner{\nabla\Lcal(\btheta),\bar{\bu}}} = \abs{\lim_{h\to0}\frac{1}{hn}\sum_{j=1}^n\p{ \ell(y_j\Phi(\btheta + h\bar{\bu};x_j)) - \ell(y_j\Phi(\btheta;x_j))}}\nonumber\\
        &= \lim_{h\to0}\frac{1}{hn}\sum_{j=1}^n\p{ \ell(y_j\Phi(\btheta;x_j)) - \ell(y_j\Phi(\btheta + h\bar{\bu};x_j))}\nonumber\\
        &\ge \lim_{h\to0}\frac{1}{hn} \p{\ell(y_i\Phi(\btheta;x_i)) - \ell(y_i\Phi(\btheta + h\bar{\bu};x_i))}\nonumber\\
        &\ge \lim_{h\to0}\frac{1}{hn} \p{\ell(y_i\Phi(\btheta;x_i)) - \ell\p{y_{i}\Phi(\btheta;x_{i}) + h\frac{\gamma qm^3}{90QM^2}}}\nonumber\\
        &= - \frac{\gamma qm^3}{90nQM^2} \ell'(y_i\Phi(\btheta;x_i)).\label{eq:grad_dot_prod_lower_bound}
    \end{align}
    In the above, the first inequality is by Cauchy-Schwartz; the second equality is due to \eqref{eq:f_non_decreasing}, which guarantees that each summand is non-positive; the second inequality is another application of \eqref{eq:f_non_decreasing} which guarantees that we're omitting only non-negative terms; and the last inequality is due to \eqref{eq:x_i_strictly_increasing}. We now lower bound the above expression depending on whether $\ell(\cdot)$ is the exponential or the logistic loss. 
    
    First assume that $\ell(\cdot)$ is the exponential loss. Then we have that $-\ell'(x)=\ell(x)$ for all $x\in\reals$, and therefore we can directly use the inequality $\ell(y_i\Phi(\btheta;x_i)) > 1/2n$.
    
    In the case where $\ell(\cdot)$ is the logistic loss, we have that
    \[
        -\ell'(x) = \frac{1}{1+\exp(x)} = 1 - \frac{1}{1+\exp(-x)}.
    \]
    By the inequality $\ell(y_i\Phi(\btheta;x_i)) > 1/2n$, we have $1+\exp(-y_i\Phi(\btheta;x_i)) > \exp(1/2n)$, implying that
    \[
        -\ell'(y_i\Phi(\btheta;x_i)) > 1-\exp\p{-\frac{1}{2n}} \ge \frac{1}{4n},
    \]
    where we used the inequality $\exp(-x)\le 1 - 0.5x$ which holds for all $x\in[0,1]$. Combining both loss cases and \eqref{eq:grad_dot_prod_lower_bound}, we arrived at
    \[
        \norm{\nabla\Lcal(\btheta)}_2 \ge  \frac{\gamma qm^3}{360n^2QM^2}.
    \]
    Squaring the above and dividing by $2$, the lemma follows.
\end{proof}

The following proposition establishes the separability (\defref{def:separability}) of a neighborhood in weight space around our initialization point (\defref{def:delta_hidden_neighborhood}) from the dataset $S$.

\begin{proposition}[Bounded Gradient with High Probability]\label{prop:PL}
    Under Assumptions~\ref{asm:init}~and~\ref{asm:data}, given any $\delta\in(0,1)$, suppose that the following hold
    \begin{equation}\label{eq:k_bound}
        k \ge 6144\cdot\frac{R^4\log\p{\frac{24r}{\delta}}}{\rho} \hskip 0.4cm \text{and}\hskip 0.4cm \Delta=\frac{\delta\rho\sigma_{\text{h}}}{24nkCR^3}.
    \end{equation}
    Then with probability at least $1-\delta$, for all $\btheta\in U_{\Delta}(\btheta(0))\cap L_{1/2n}^{+}(\Lcal)$, we have that 
    \[
        \frac12\norm{\nabla\Lcal(\btheta)}_2^2 \ge 3\cdot10^{-11}\frac{\delta^2\rho^2}{n^6r^2C^2R^8}\sigma_{\text{h}}^2.
    \]

    
\end{proposition}

\begin{proof}
    To prove the proposition, we will show that for all $\btheta\in U_{\Delta}(\btheta(0))\cap L_{1/2n}^{+}(\Lcal)$, $\btheta$ is separable from $S$ with high probability. By \lemref{lem:grad_norm_lbound}, this would imply the proposition. We will begin with proving items 1-4 jointly, and then we will show item 5 separately.
    \begin{enumerate}
        \item[1-4.]
        Under \asmref{asm:init}, suppose we are given some $a\in[0,R]$ and $\bxi\coloneqq(\xi_1,\xi_2)\in\{-1,1\}^2$, and sufficiently small $\varepsilon>0$ such that
        \begin{equation}\label{eq:rho_bound}
            \varepsilon \le \frac{512R^4\log\p{\frac{24r}{\delta}}}{k} \le \frac{\rho}{6} \le \frac{R}{6},
        \end{equation}
        where the first inequality is by assumption on $\varepsilon$, the second inequality is by the lower bound on $k$ in \eqref{eq:k_bound}, and the last inequality follows from $\rho\le\frac{2R}{r+1}\le R$ since $\rho$ must be smaller than the average length of an interval and since we assume $r\ge1$. We now consider the event denoted by $E_{a,\varepsilon,\bxi}$ where the weights $w_i,b_i$ of the $i$-th neuron satisfy
        \begin{equation}\label{eq:event_E}
            w_i\in\p{\xi_1\frac{\sigma_{\text{h}}}{ R}, \xi_1\frac{2\sigma_{\text{h}}}{ R}}\hskip 0.3cm\text{and}\hskip 0.3cm -\frac{b_i}{w_i}\in(\xi_2a,\xi_2(a+\varepsilon)),
        \end{equation}
        Since such an event is symmetric about $0$, it is unaffected by the signs of $\bxi$. We can therefore assume without loss of generality that both intervals in \eqref{eq:event_E} are contained in the positive real line and omit $\bxi$ from our notation. Under this assumption, the probability of $E_{a,\varepsilon}$ can be given in terms of Owen's T function which is defined by
        \[
            T(h,a)\coloneqq\frac{1}{2\pi}\int_{0}^{a}\frac{\exp\p{-\frac12h^2(1+x^2)}}{1+x^2}dx
        \]
        (see \citet{owen1956tables}), yielding
        \[
            \pr[E_{a,\varepsilon}]=T\p{\frac{1}{ R},a+\varepsilon} - T\p{\frac{1}{ R},a} - \p{T\p{\frac{2}{ R},a+\varepsilon} - T\p{\frac{2}{ R},a}}.
        \]
        Using the definition of $T(\cdot,\cdot)$, the above can be simplified to
        \begin{align}
            \pr[E_{a,\varepsilon}] &= \frac{1}{2\pi}\int_{a}^{a+\varepsilon}\frac{\exp\p{-\frac{1}{2 R^2}(1+x^2)}\p{1-\exp\p{-\frac{3}{2 R^2}(1+x^2)}}}{1+x^2}dx \nonumber\\
            &\ge \frac{1}{2\pi}\p{1-\exp\p{-\frac{3}{2 R^2}}}\int_{a}^{a+\varepsilon}\frac{\exp\p{-\frac{1}{2R^2}(1+x^2)}}{1+x^2}dx \nonumber\\
            &\ge \frac{3}{8\pi R^2}\int_{a}^{a+\varepsilon}\frac{\exp\p{-\frac{1}{2 R^2}(1+x^2)}}{1+x^2}dx \nonumber\\
            &\ge \frac{3}{8\pi R^2}\int_{a}^{a+\varepsilon}\frac{\exp\p{-\frac{1}{2R^2}(1+( R+\varepsilon)^2)}}{1+( R+\varepsilon)^2}dx \ge \frac{\varepsilon}{512 R^4}.\label{eq:E_pr}
        \end{align}
        In the above, the second inequality follows from the inequality $1-\exp(-x)\ge 0.5x$ which holds for all $x\in[0,1.5]$ and from the fact that $1\le  R$; the third inequality follows from the fact that the integrand is a monotonically decreasing function and $|a|\le R$; and the last inequality follows from $\varepsilon\le  R$ which is implied by \eqref{eq:rho_bound} and allows us to lower bound the numerator of the integrand by $\exp(-2.5)$ and upper bound the denominator by $5 R^2$, and the fact that $3/(40\pi\exp(2.5))\ge1/512$.
        
        Next, given some interval $I_j\coloneqq(x_{i_{j}}, x_{i_{j+1}+1})$, $j\in[r+1]$, where the classification does not change signs on the data, we consider the three sub-intervals given by
        \begin{align*}
            I_{j_1}&\coloneqq(x_{i_{j}},x_{i_{j}} + \varepsilon),\\
            I_{j_2}&\coloneqq \p{\frac{x_{i_{j}}+x_{i_{j+1}+1}}{2}-\frac{\varepsilon}{2},\frac{x_{i_{j}}+x_{i_{j+1}+1}}{2}+\frac{\varepsilon}{2}},\\
            I_{j_3}&\coloneqq \p{x_{i_{j+1}+1} - \varepsilon,x_{i_{j+1}+1}}.
        \end{align*}
        We remark that due to \eqref{eq:rho_bound}, the above sub-intervals are all disjoint and the distance between the intervals is positive. We now wish to show that Items~\ref{item:1}~and~\ref{item:3} hold. We have from \eqref{eq:E_pr} that the probability that a given sub-interval of length $\varepsilon$ contains no breakpoint is at most
        \[
            (1-\pr[E_{a,\varepsilon}])^{k} \le \p{1-\frac{\log\p{\frac{24r}{\delta}}}{k}}^k \le \exp\p{-\log\p{\frac{24r}{\delta}}} = \frac{\delta}{24r},
        \]
        where we used the inequality $(1-x/y)^y\le\exp(-x)$ which holds for all $x,y>0$. There are exactly $3\cdot(r+1)\le6r$ sub-intervals, therefore by a union bound we have that Items~\ref{item:1}~and~\ref{item:3} hold for some positive $q,Q,m,M$ with probability at least $1-\frac{\delta}{4}$.
        
        Next, we show \itemref{item:2}. By \eqref{eq:E_pr}, we have
        \[
            \pr\pcc{E_{ R, R/6}} \ge \frac{1}{3072 R^3}.
        \]
        Thus, the probability of initializing a neuron with breakpoint in $( R,\frac{7}{6} R)$ which is active on all the data points is at least $\frac{1}{6144 R^3}$, since there's an independent $0.5$ probability that it has the correct orientation. This entails that the probability of initializing at most one neuron which is active on all the data points and has a breakpoint in $( R,\frac76 R)$ is upper bounded by
        \begin{align*}
            &\p{1-\pr\pcc{E_{ R, R/6}}}^k + k\p{1-\pr\pcc{E_{ R, R/6}}}^{k-1}\pr\pcc{E_{ R, R/6}} \le 2\p{1-\frac{\pr\pcc{E_{ R, R/6}}}{2}}^k\\
            &\hskip 2cm\le 2\p{1-\frac{1}{6144 R^3}}^k \le 2\exp\p{-\frac{k}{6144 R^3}}\\
            &\hskip 2cm\le 2\exp\p{-\frac{ R\log(24r/\delta)}{\rho}} \le \frac{2\delta}{24r} \le \frac{\delta}{12}.
        \end{align*}
        In the above, the first inequality follows from the inequality $(1-x)^k+k(1-x)^{k-1}x \le 2(1-x/2)^k$ which holds for any natural $k$ and all $x\in[0,1]$,\footnote{To show this inequality holds, consider $k$ i.i.d.\ random variables $X_j\sim U([0,1])$. Then the left-hand side equals $\pr[|\set{x_j:x_j\in[0,x]}|\le1]$. The occurrence of the complement of this event is implied if $x_{i}\in[0,x/2]$ and $x_{i'}\in[x/2,x]$ hold for some $i\neq i'$, therefore to upper bound the left-hand side it suffices to upper bound the complement of the event where $x_{i}\in[0,x/2]$ and $x_{i'}\in[x/2,x]$ hold for some $i\neq i'$. This in turn follows from applying a union bound on $\pr[|\set{x_j:x_j\in[0,x/2]}|=0]$ and $\pr[|\set{x_j:x_j\in[x/2,x]}|=0]$.} the third inequality follows from $(1-1/x)^x\le\exp(-1)$ for all $x>0$, the fourth inequality follows from our lower bound on $k$ in \eqref{eq:k_bound}, and the penultimate inequality holds due to \eqref{eq:rho_bound} which entails $\rho\le R$. Therefore, by the above and a union bound on the symmetric event where two neurons are initialized in $(-\frac76 R,- R)$, we have that \itemref{item:2} holds for some $m,M$ with probability at least $1-\frac{\delta}{6}\ge1-\frac{\delta}{4}$.
        
        We will now derive explicit bounds on the constants $m,M,q,Q$, and in addition we will show that \itemref{item:4} holds. Applying a union bound on the two previous cases, we have that Items~\ref{item:1}-\ref{item:3} hold with probability at least $1-\delta/2$. In such a case, we get an explicit lower bound on $q$ as follows
        \[
            q\ge \frac{|I_{j}|-3\varepsilon}{2} \ge\frac{\rho-3\varepsilon}{2} \ge  \frac{\rho}{4},
        \]
        where in the second inequality we used the fact that $|I_{j}|\ge\rho$ for all $j\in[r+1]$ which holds by the definition of $\rho$ and in the last inequality we used \eqref{eq:rho_bound}. To bound $Q$ in \itemref{item:4}, we first argue that under the realization of $E_{ R, R/6}$, the four neurons that are active on all the data points have a breakpoint with absolute value at most $\frac{7}{6} R\le2 R$. To upper bound $Q$ in \itemref{item:3}, observe that under the realization of the previous events we have that
        \[
            \beta_{i+1}(I_j) - \beta_{i}(I_j) \le \frac{|I_{j}|+\varepsilon}{2} \le R+\frac{\varepsilon-\rho}{2} \le R,
        \]
        for all $j\in[r+1]$ and $i\in[2]$, which follows from $|I_{j}|\le2R-\rho$ since $r\ge1$ and $I_j\subseteq [-R,R]$ (i.e.\ there exists at least one interval other than $I_j$ which has length at least $\rho$), and from the inequality $\varepsilon\le\rho/6$ which holds by \eqref{eq:rho_bound}. To upper bound $Q$ in \itemref{item:4}, we bound the term $|\beta_1(I_{j+1})-\beta_3(I_j)|$. Observe that under the realization of the above events we have that $\beta_1(I_{j+1})\in(x,x+\varepsilon)$ and $\beta_3(I_j)\in(x'-\varepsilon,x')$ where $x<x'$ are the largest and smallest data instances in $I_{j},I_{j+1}$, respectively. We therefore have
        \[
            |\beta_1(I_{j+1})-\beta_3(I_j)| \le \max\{|x-x'|,|x-x'+2\varepsilon|\} \le |x-x'| + 2\varepsilon \le \frac73 R,
        \]
        where the second inequality follows from the triangle inequality and the last inequality follows from the fact that $x_i\in[-R,R]$ for all $i\in[n]$ and from \eqref{eq:rho_bound}. Turning to bound $w_i,b_i$, we have by Eqs.~(\ref{eq:rho_bound},\ref{eq:event_E}) that when $E_{a,\varepsilon}$ or $E_{ R, R/6}$ occur then the $i$-th neuron satisfies $|w_i|\ge \sigma/ R$ and
        \[
            |b_i|\le( R+\frac16 R)w_i\le \frac76 R\cdot\frac{2\sigma_{\text{h}}}{ R}
            = \frac73\sigma_{\text{h}},
        \]
        concluding the derivation of Items~\ref{item:1}-\ref{item:4}.
        \item[5.]
        It will suffice to lower bound the probability of the event denoted by $A$ where $x_i\notin(\alpha_j-\gamma/2,\alpha_j+\gamma/2)$ for all $i\in[n]$ and $j\in[r]$, where $\alpha_j$ is the $j$-th sign change of the ground truth function labelling $y_i$. The set $\cup_{j=1}^r(\alpha_j-\gamma/2,\alpha_j+\gamma/2)$ has Lebesgue measure of at most $\gamma r$, and since by \asmref{asm:data} we have that $\mu(x)\le C$ for all $x\in[-R,R]$, we lower bound the probability of the event by the expression
        \[
            \pr[A] \ge (1-\gamma rC)^n.
        \]
        Plugging $\gamma=\frac{\delta}{4nrC}$ in the above which entails $\gamma rC\le1$ and using Bernoulli's inequality we have
        \[
            \pr[A] \ge 1-\frac{\delta}{4}.
        \]

    \end{enumerate}
    To conclude the derivation so far, using another union bound, we have shown that with probability at least $1-0.75\delta$, $\btheta(0)$ is separable from $S$ with constants
        \begin{equation}\label{eq:sep_zero}
            \gamma_0=\frac{\delta}{4nrC},\hskip 0.3cm q_0=\frac{\rho}{4},\hskip  0.3cm Q_0= \frac73 R,\hskip 0.3cm m_0=\frac{\sigma_{\text{h}}}{ R},\hskip 0.3cm M_0=\frac73\sigma_{\text{h}}.
        \end{equation}
        We will now show that the separability also holds in a $\Delta$-hidden neighborhood of $\btheta(0)$ for an appropriately chosen $\Delta>0$. To this end, we first establish that
        \[
            \pr\pcc{\min_{i\in[n],j\in[k]}|\beta_j-x_i|> \frac{\delta}{8nkC}} \ge 1-\frac{\delta}{4}.
        \]
        Suppose we have $n$ data instances in $[-R,R]$, then a cover of radius $\frac{\delta\pi}{8kn}$ over these data points has a (one-dimensional Lebesgue) measure of at most $\frac{\delta\pi}{4k}$. Thus, the probability that a breakpoint will not be initialized within distance less than $\frac{\delta\pi}{8kn}$ from any point is at least $(1-\frac{\delta}{4k})$. This is true since \asmref{asm:init} implies that the distribution of a breakpoint is a standard Cauchy distribution with density at most $1/\pi$. We thus have that
        \begin{equation}\label{eq:diff_bound}
            \pr\pcc{\min_{i\in[n],j\in[k]}|\beta_j-x_i|>\frac{\delta}{8nkC}} \ge \p{1-\frac{\delta}{4k}}^k \ge 1-\frac{\delta}{4},
        \end{equation}
        where the last inequality follows from Bernoulli's inequality. A final union bound now implies that the above bound holds with the previous implications with probability at least $1-\delta$. Define
        \[
            \Delta\coloneqq\frac{\delta\rho\sigma_{\text{h}}}{24nkC R^3},
        \]
        we will now show that this implies the uniform separability of any $\btheta\in U_{\Delta}(\btheta(0))$ from $S$, by proving Items~\ref{item:1}-\ref{item:4} jointly and \itemref{item:5} separately.
        \begin{enumerate}
            \item[1-4.]
            First, by \asmref{asm:data}, we have $1=\int_{-R}^{R}\mu(x)dx\le 2RC$ which with \eqref{eq:rho_bound} implies that $\Delta \le \frac{\delta\sigma_{\text{h}}}{24nk R}\le\frac{\sigma_{\text{h}}}{24 R}$. Since the weight and bias of each neuron in $U_{\Delta}(\btheta(0))$ change by at most $\Delta$, we have
            \[
                m\ge \frac{\sigma_{\text{h}}}{ R} - \frac{\sigma_{\text{h}}}{24 R} = \frac{23\sigma_{\text{h}}}{24 R} \hskip 0.4cm \text{and} \hskip 0.4cm M\le \frac73\sigma_{\text{h}} + \frac{\sigma_{\text{h}}}{24 R} \le \frac{57}{24}\sigma_{\text{h}}.
            \]
            To bound $q$ and $Q$, we will first show that under our assumptions the breakpoints cannot move much. To this end, we show that for each neuron, the function $f(w,b)\coloneqq-\frac{b}{w}$ is Lipschitz on $U_{\Delta}(\btheta(0))$. We have
            \[
                \nabla f(w,b)=\p{\frac{b}{w^2},-\frac{1}{w}},
            \]
            and therefore for any neuron $(w,b)\in\btheta$ such that $\btheta\in U_{\Delta}(\btheta(0))$ we get
            \[
                \norm{\nabla f(w,b)} = \sqrt{\frac{b^2}{w^4}+\frac{1}{w^2}} \le
                \sqrt{\frac{M^2}{m^4}+\frac{1}{m^2}} \le \frac{24 R}{23\sigma_{\text{h}}}\sqrt{\frac{57^2}{23^2} R^2+1} < \frac{3 R^2}{\sigma_{\text{h}}},
            \]
            where the last inequality follows from $1\le R^2$. This implies that
            \[
                \abs{-\frac{b}{w}+\frac{b_0}{w_0}} \le \norm{\nabla f(w,b)}\cdot\norm{(w,b)-(w_0,b_0)} < \frac{3 R^2}{\sigma_{\text{h}}}\Delta \le \frac{\delta\rho}{8nkCR}.
            \]
            That is, we have that the breakpoint of each neuron moves a distance strictly less than $\frac{\delta\rho}{8nkC R}\le\frac{\delta}{8nkC}$, which along with \eqref{eq:diff_bound} guarantees that $\Lcal(\cdot)$ is differentiable on $U_{\Delta}(\btheta(0))$ since no ReLU crosses a data instance. Since $Q$ is the upper bound on the difference between two breakpoints where each moves by at most $\frac{\delta}{8nkC}$, this also yields a bound on $Q$ as follows
            \[
                Q \le Q_0 + 2\frac{\delta}{8nkC} \le \frac73 R +  \frac{1}{8} R \le 2.5 R,
            \]
            where we used the upper bound on $Q_0$ from \eqref{eq:sep_zero}, \eqref{eq:k_bound} which implies $k\ge4$ (since $\rho\le R$ by \eqref{eq:rho_bound}), and $1/C\le2 R$. Likewise, to lower bound $q$, compute
            \[
                q \ge q_0 - 2\frac{\delta\rho}{8nkC R} \ge \frac{\rho}{4} - \frac{\rho}{20 R} \ge \frac{\rho}{5},
            \]
            where again we used \eqref{eq:sep_zero}, \eqref{eq:k_bound} which implies $k\ge10$, and $1/C\le2 R$.
            \item[5.]
            Since $\gamma$ depends on $S$ and not on $\btheta$, it remains unchanged and we have $\gamma=\gamma_0$. 
        \end{enumerate}
        We can now use the assumption $\btheta\in L_{1/2n}^{+}(\Lcal)$ and \lemref{lem:grad_norm_lbound} to conclude
        \[
            \frac12\norm{\nabla\Lcal(\btheta)}_2^2 \ge \frac{\gamma^2q^2m^6}{259200n^4Q^2M^4} \ge \frac{1}{259200n^4}\cdot\frac{\delta^2}{4^2n^2r^2C^2}\cdot\frac{4\rho^2}{25^2 R^2}\cdot\frac{23^6}{57^424^2 R^6}\cdot\sigma_{\text{h}}^2.
        \]
        Simplifying the above, the proposition follows.
        
\end{proof}

Having established the required machinery for proving \thmref{thm:optimization}, we now turn to do so.

\begin{proof}[Proof of \thmref{thm:optimization}]\label{app:proof_optimization}
    We begin with bounding the loss upon initialization with high probability. First, consider $3k$ i.i.d.\ random variables $X_j\sim\Ncal(0,1)$. We have that
    \[
        \pr\pcc{\max_{j\in[3k]}|X_j|\le x} = \p{\erf\p{\frac{x}{\sqrt{2}}}}^{3k} \ge \p{1-\exp\p{-0.5x^2}}^{3k} \ge 1-3k\exp\p{-0.5x^2},
    \]
    where the first inequality follows from $1-\erf(x)<\exp(-x^2)$ for all $x\ge0$ (see Eq.~(7.8.3) in \citet{NIST:DLMF}) and the second inequality follows from Bernoulli's inequality since $\exp(-0.5x^2)<1$. Plugging $x=\sqrt{2\log(6k/\delta)}$ in the above, we have
    \[
        \pr\pcc{\max_{j\in[3k]}|X_j|\le \sqrt{2\log(6k/\delta)}} \ge 1-\frac{3k\delta}{6k} = 1-\frac{\delta}{2}.
    \]
    Thus, with probability at least $1-\frac{\delta}{2}$, we have that all the weights of $\btheta(0)$ are at most $\sqrt{2\log(6k/\delta)}$ standard deviations away from zero. With this bound, we can derive for all $x\in[-R,R]$
    \[
        \Ncal_{\btheta(0)}(x) \le \sum_{j \in [k]}|v_j| \sigma(|w_j| \cdot |x| + |b_j|) \le 4kR\sigma_{\text{h}}\sigma_{\text{o}}\log\p{\frac{6k}{\delta}},
    \]
    which for both the exponential and logistic losses implies
    \begin{equation}\label{eq:bounded_loss_whp}
        \Lcal(\btheta(0)) \le \exp\p{4kR\sigma_{\text{h}}\sigma_{\text{o}}\log\p{\frac{6k}{\delta}}} \le e,
    \end{equation}
    where the last inequality is by our assumption $\sigma_{\text{o}}\le\frac{1}{4kR\sigma_{\text{h}}\log\p{\frac{6k}{\delta}}}$. Letting
    \[
        \lambda\coloneqq 10^{-11}\frac{\delta^2\rho^2}{n^6r^2C^2R^8}
    \]
    and observing that our lower bound assumption on $\sigma_{\text{h}}$ in \eqref{eq:convergence_assumptions} implies it's at least $0.5$ since $\rho\le R$ and $C\ge1/2 R$, we can invoke \propref{prop:PL} with confidence $\frac{\delta}{2}$ to obtain
    \begin{align}
        \frac12\norm{\nabla\Lcal(\btheta(t))}_2^2 &\ge 3\cdot10^{-11}\frac{\delta^2\rho^2}{n^6r^2C^2R^8}\sigma_{\text{h}}^2\nonumber\\
        &\ge 3\cdot10^{-11}\frac{\delta^2\rho^2}{n^6r^2C^2R^8}\sigma_{\text{h}}^2\cdot\frac{\Lcal(\btheta(t))}{\Lcal(\btheta(0))}\nonumber\\
        &\ge \lambda\sigma_{\text{h}}^2\cdot\Lcal(\btheta(t)),\label{eq:PL}
    \end{align}
    where the second inequality holds since $\frac{\Lcal(\btheta(t))}{\Lcal(\btheta(0))}\le1$ because the flow is non-increasing, and the last inequality holds due to \eqref{eq:bounded_loss_whp} which implies $\frac{1}{\Lcal(\btheta(0))}\ge\exp(-1)\ge\frac{1}{3}$. By a union bound, the above holds with probability at least $1-\delta$. 
    
    Denote $D\coloneqq U_{\Delta}(\btheta(0))\cap L_{\frac{1}{2n}}^{+}(\Lcal)$ where $\Delta$ is defined in \eqref{eq:k_bound}, and define $t'\in[0,\infty)$ to be the smallest time such that $\btheta(t')$ is on the boundary of $U_{\Delta}(\btheta(0))$ 
     (where $t'=\infty$ if there exists no such time). We will now show that the flow attains loss at most $\frac{1}{2n}$ in time $t_0\coloneqq\frac{\log(2n\Lcal(\btheta(0)))}{2\lambda\sigma_{\text{h}}^2}$, by analyzing several different cases.
     \begin{itemize}
        \item
        Suppose that $t'>t_0$.
        \begin{itemize}
            \item
            If $\set{\btheta(t):t\in[0,t_0]}\subseteq D$, then by the PL-condition shown in \eqref{eq:PL} we have for all $t\in[0,t_0]$ that GF enjoys a convergence rate of
            \[
                \Lcal(\btheta(t)) \le \exp\p{-2\lambda\sigma_{\text{h}}^2 t } \cdot\Lcal(\btheta(0)).
            \]
            Plugging $t=t_0$ in the above and simplifying, we have $\Lcal(\btheta(t_0))\le\frac{1}{2n}$.
            \item
            If $\set{\btheta(t):t\in[0,t_0]}\not\subseteq D$, then there exists a time $t''\le t_0$ such that $\btheta(t'')\notin D$. Since $t''\le t_0<t'$, it must hold that $t''\notin L_{\frac{1}{2n}}^{+}(\Lcal)$, and therefore $\Lcal(\btheta(t''))<\frac{1}{2n}$ which implies $\Lcal(\btheta(t_0))<\frac{1}{2n}$ since the flow is non-increasing.
        \end{itemize}
        \item
        Suppose that $t'\le t_0$. Assume by contradiction that $\set{\btheta(t):t\in[0,t']}\subseteq D$. We will now show that the length of the trajectory of GF cannot have been long enough to reach the boundary of $D$, which will result in a contradiction. To this end, we use a similar technique as in \citet[Thm.~9]{gupta2021path}. Define the potential function $\varepsilon(t)=\sqrt{\Lcal(\btheta(t))}$. Taking the derivative of $\varepsilon(t)$ with respect to $t$ and using the chain rule we have
         \[
            \dot{\varepsilon}(t) = \frac{\frac{d\Lcal(\btheta(t))}{dt}}{2\sqrt{\Lcal(\btheta(t))}} = - \frac{\norm{\nabla\Lcal(\btheta(t))}_2^2}{2\sqrt{\Lcal(\btheta(t))}} \le - \sqrt{\frac{\lambda}{2}}\sigma_{\text{h}}\cdot\norm{\nabla\Lcal(\btheta(t))}_2,
         \]
         where the inequality follows from \eqref{eq:PL}.
         We can now bound the length of the trajectory up until time $t'$ by using the fundamental theorem of calculus and obtain
         \begin{align*}
            \int_0^{t'}\norm{\nabla\Lcal(\btheta(t))}_2dt &\le -\frac{1}{\sigma_{\text{h}}}\sqrt{\frac{2}{\lambda}}\int_{0}^{t'}\dot{\varepsilon}(t)dt \le -\frac{1}{\sigma_{\text{h}}}\sqrt{\frac{2}{\lambda}} \pcc{\sqrt{\Lcal(\btheta(t))}}_{0}^{t'}\\
            &\le \frac{1}{\sigma_{\text{h}}}\sqrt{\frac{2\Lcal(\btheta(0))}{\lambda}} \le \frac{1}{\sigma_{\text{h}}}\sqrt{\frac{2e}{\lambda}} < \Delta,
         \end{align*}
         where in the second line, the first inequality uses the fact that $\Lcal(\cdot)>0$, the second inequality follows from \eqref{eq:bounded_loss_whp}, and the last inequality follows from our bound on $\sigma_{\text{h}}$ assumed in  \eqref{eq:convergence_assumptions} and the definition of $\Delta$ in \eqref{eq:k_bound}. In contrast, since $\btheta(t')$ is on the boundary of $U_{\Delta}(\btheta(0))$, this implies that there exists some neuron with weight and bias $w(t),b(t)$ at time $t\ge0$ such that $\norm{(w(t'),b(t')) - (w(0),b(0))}_2 =\Delta$. From this and the path length upper bound we have
         \[
            \Delta \le \norm{\btheta(t')-\btheta(0)}_2 \le \int_0^{t'}\norm{\nabla\Lcal(\btheta(t))}_2dt < \Delta,
         \]
         which is a contradiction. We therefore must have that $\set{\btheta(t):t\in[0,t']}\not\subseteq D$. When this holds, there exists a time $t''\le t'\le t_0$ such that $\btheta(t'')\notin D$. Since $t''\le t'$, it must hold that $t''\notin L_{\frac{1}{2n}}^{+}(\Lcal)$, and therefore $\Lcal(\btheta(t''))<\frac{1}{2n}$ which implies $\Lcal(\btheta(t_0))<\frac{1}{2n}$ since the flow is non-increasing.
     \end{itemize}


\end{proof}

\section{Over-parameterization is necessary}\label{app:necessary}
    
    In this appendix, we further discuss and formally prove \thmref{thm:dormant_neuron}, which establishes that in general under \asmref{asm:init}, an over-parameterization of magnitude at least $1.3r$ is necessary for achieving population loss below a constant. Our analysis is based on the following specific construction, where the labels are determined by a function $f_r$ parameterized by a natural number $r$ for all $x\in[-1,1]$, expressible by the sign of a teacher network of width $r$ and defined as
    \begin{equation}\label{eq:f_r_def}
        f_r(x)\coloneqq \sign(\sin(0.5\pi(r+1)(x+1))).
    \end{equation}
    That is, $f_r$ changes value $r$ times between $-1$ and $1$ on the interval $[-1,1]$, and is constant along intervals of length $\frac{2}{r+1}$. We now define the distribution $\Dcal$ over the inputs of the dataset used in our lower bound and its corresponding labelling rule. We have
    \begin{equation}\label{eq:D_def}
        x\sim U[-1,1]\hskip 0.3cm\text{and}\hskip 0.3cm y=f_r(x).
    \end{equation}
    
    Recall the statement of \thmref{thm:dormant_neuron}, we have for example that if $\alpha=1.3$, then the width of the network being trained is no more than $1.3r$ and GF attains loss at least $\frac{1}{160}$ in this case. While a lower bound of $r$ neurons for the construction specified in \eqref{eq:D_def} is trivially implied by function approximation considerations, our lower bound merely improves upon this quantity by a constant multiplicative factor. Nevertheless, it is interesting to compare our lower bound to other similar settings in the literature, since it is typically difficult to derive lower bounds that require strictly more than $r$ neurons. For example, in a teacher-student setting where networks of the form $\bx\mapsto\sum_{i=1}^r\relu{\bw_i^{\top}\bx}$ are considered, it is known that there are spurious (non-global) minima already when $r\ge6$ \citep{safran2018spurious,arjevani2020analytic,arjevani2021analytic}, and that empirically we are more likely to get stuck in those minima the larger $r$ is \citep{safran2018spurious}, but in spite of this ample empirical evidence, there is no proof that optimization will fail for any natural number $r\ge6$. In contrast, in our univariate setting, it is possible to show a non-trivial lower bound since we utilize bias terms. This highlights the difference between settings that omit and include biases, which impacts the associated optimization problem in a non-trivial manner.
    
    The proof of our lower bound, which appears below in \appref{app:proof_dormant_neurons}, relies on the observation that under \asmref{asm:init}, the breakpoints of the trained network upon initialization follow a standard Cauchy distribution. In such a case, neurons with a breakpoint outside the support of the data and with the wrong orientation will remain dormant throughout the optimization process, which requires initializing at least a fraction more of the minimal number of neurons required so that sufficiently many will be optimized and could improve the approximation of the target function. While one can circumvent this issue by scaling the breakpoints to the support of the data, this would require (i) an initialization scheme which is different than \asmref{asm:init}, which is used in our upper bounds; and (ii) this may even prove detrimental to optimization, as our positive result requires neurons that are active on all the data instances. We stress that our lower bound given here applies to training over a sample of any size, since it relies on approximation arguments. Additionally, we remark that by scaling the distribution $\Dcal$ to be supported on a smaller interval we can increase the required magnitude of over-parameterization up to a factor of $\alpha=2$, however due to the common practice of scaling the data to have unit norm, we assume it is supported on $[-1,1]$. We also remark that a common initialization scheme is to set the bias terms to zero \citep{he2015delving}. This results in breakpoints that are initialized at the origin and circumvents the issue of dormant neurons upon initialization, however the main motivation for using such an initialization scheme is to achieve numerical stability and avoid exploding gradients when training very deep networks, which is not an issue for the shallow architecture we consider here. In any case, we stress that the goal of our lower bounds is to exemplify that over-parameterization is necessary in a setting complementary to our upper bound in \thmref{thm:optimization}, and we leave the derivation of stronger lower bounds under more general initialization schemes as a tantalizing future work direction.

    \subsection{Proof of \thmref{thm:dormant_neuron}}\label{app:proof_dormant_neurons}

To prove the theorem, we would need the following auxiliary lemmas. The first lemma below establishes that if we approximate the function $f_r$ which is defined in \eqref{eq:f_r_def} by a function which does not change its sign over an interval of length larger than $\frac{2}{r+1}$, then this results in a strictly positive loss which is roughly proportional to the length of the approximation interval.

\begin{lemma}\label{lem:f_r_lower_bound}
        Let $f_r$ be as defined in \eqref{eq:f_r_def}, let $\beta_1,\beta_2\in[-1,1]$, $\Ncal:[\beta_1,\beta_2]\to\reals$ such that its sign is fixed on $[\beta_1,\beta_2]$, and let $\ell$ be either the exponential or logistic loss. Then
        \[
            \int_{\beta_1}^{\beta_2}\ell(\Ncal(x)\cdot f_r(x))dx \ge\frac12\ell(0)\p{\beta_2-\beta_1-\frac{2}{r+1}}.
        \]
    \end{lemma}
    
    \begin{proof}
        If $\beta_2-\beta_1\le\frac{2}{r+1}$, then the right-hand side is non-positive and the lemma follows since $\ell(\cdot)>0$. If $\beta_2-\beta_1>\frac{2}{r+1}$, then the (one-dimensional Lebesgue) measure of the set $A\coloneqq\{x\in[\beta_1,\beta_2]:\Ncal(x)\cdot f_r(x)\le0\}$ is at least
        \[
            \frac12\p{\beta_2-\beta_1-\frac{2}{r+1}},
        \]
        since the measure of the complementary set $\{x\in[\beta_1,\beta_2]:\Ncal(x)\cdot f_r(x)>0\}$ is at most $\frac12\p{\beta_2-\beta_1+\frac{2}{r+1}}$, where the upper bound is attained when $\Ncal(x)\cdot f_r(x)>0$ for all
        \[
            x\in\pcc{\beta_1,\beta_1+\frac{2}{r+1}}\cup\pcc{\beta_2-\frac{2}{r+1},\beta_2}.
        \]
        We can therefore lower bound the integral in the lemma by
        \[
            \int_{\beta_1}^{\beta_2}\ell(\Ncal(x)\cdot f_r(x))dx \ge \int_{x\in A}\ell(\Ncal(x)\cdot f_r(x))dx \ge \int_{x\in A}\ell(0)dx \ge \frac12\ell(0)\p{\beta_2-\beta_1-\frac{2}{r+1}}.
        \]
    \end{proof}
    
    The following lemma shows that approximating $f_r$ using a ReLU network with just $r'$ neurons results in loss proportional to $1-r'/r$.
    
    \begin{lemma}\label{lem:telgarsky_lower_bound}
        Suppose that $f_r$ as defined in \eqref{eq:f_r_def}. Then for any ReLU network $\Ncal_{\btheta}$ of width at most $r'$, we have
        \[
            \Lcal_{\Dcal}(\btheta) \ge \frac14\p{1-\frac{r'}{r}}.
        \]
    \end{lemma}
    
    \begin{proof}
        Denote by $\beta_1,\ldots,\beta_{r'}$ the set of points where $\Ncal$ changes sign in $(-1,1)$. Note that this set is always of size at most $r'$, and we may assume without loss of generality that it is of size exactly $r'$ (since otherwise we can prove a stronger claim, where the lemma holds for some $r''<r'$). Further define the boundaries $\beta_0\coloneqq-1$ and $\beta_{r'+1}\coloneqq1$. We compute
        \begin{align*}
            \Lcal_{\Dcal}(\btheta) &= \int_{-1}^1\frac12\ell(f_r(x)\cdot\Ncal_{\btheta}(x))dx = \frac12\sum_{i=0}^{r'} \int_{\beta_{i}}^{\beta_{i+1}}\ell(f_r(x)\cdot\Ncal_{\btheta}(x))dx\\
            &\ge \frac14\ell(0)\sum_{i=0}^{r'}\p{\beta_{i+1}-\beta_i-\frac{2}{r+1}} = \frac14\p{2-2\frac{r'+1}{r+1}}\\
            &=\frac12\cdot\frac{r-r'}{r+1} \ge\frac14\p{1-\frac{r'}{r}},
        \end{align*}
        where the first inequality uses \lemref{lem:f_r_lower_bound}, the equality that follows is due to the sum telescoping and since $\ell(0)=1$ for both the exponential and logistic losses, and the final inequality is due to $r\ge1$.
    \end{proof}
    
    With the above auxiliary lemmas, we can now turn to the proof of the theorem.
    
    \begin{proof}[Proof of \thmref{thm:dormant_neuron}]
        By \asmref{asm:init}, the breakpoints of $\Ncal$ at initialization follow a standard Cauchy distribution. Since $\Dcal$ is supported on $[-1,1]$, we have with probability exactly $0.5$ that the breakpoint of a given neuron falls outside of $[-1,1]$. Moreover, with an independent probability of $0.5$, the orientation of the neuron is such that it is off on all the data instances. I.e., such a neuron remains dormant throughout the optimization process of GF with probability $0.25$. It can be verified that for any integer $k\ge1$, at least $\lceil0.25k\rceil$ neurons will be dormant upon initialization with probability at least $0.25$.\footnote{Essentially, this holds true since the median of a binomially-distributed random variable $B(n,k)$ is $\lceil nk\rceil$ or $\lfloor nk\rfloor$, and since deviating from the median by at most $1$ never increases the probability of the tail to more than $0.75$.} Thus, with probability at least $0.25$, we have that out of $\alpha r$ neurons, there are at most $r'=\lfloor0.75\alpha r\rfloor$ neurons that are effectively being trained with breakpoints in $[-1,1]$. By \lemref{lem:telgarsky_lower_bound}, this results in a lower bound on the population loss of
        \[
            \Lcal_{\Dcal}(\Ncal_{\btheta(t)}) \ge \frac14\p{1-\frac{\lfloor0.75\alpha r\rfloor}{r}} \ge \frac14\p{1-0.75\alpha},
        \]
        for any time $t\ge0$.
    \end{proof}

\section{Proof of \thmref{thm:minimize regions}} \label{app:proof of minimize regions}

By \thmref{thm:known KKT}, if there exists time $t_0$ such that $\cl(\btheta(t_0)) < \frac{1}{n}$ then GF converges to zero loss, and converges in direction to a KKT point of Problem~(\ref{eq:optimization problem}).
We denote $\cn_\btheta(x) = \sum_{j \in [k]} v_j \sigma(w_j x + b_j)$. Thus, $\cn_\btheta$ is a network of width $k$, where the weights in the first layer are $w_1,\ldots,w_k$, the bias terms are $b_1,\ldots,b_k$, and the weights in the second layer are $v_1,\ldots,v_k$. 
We denote $J := \{j \in [k]: v_j \neq 0\}$, $J^+ := \{j \in J: v_j > 0\}$, and $J^- := \{j \in J: v_j < 0\}$. Since neurons with output weight $v_j = 0$ do not affect the function that the network computes, then in this proof we ignore them.
We also denote $I := [n]$, and $I' = \{i \in I : y_i \cn_\btheta(x_i) = 1 \}$. Thus, $I'$ are the indices of the examples where $\cn_\btheta$ attains margin of exactly $1$. 

Assume that $\cn_\btheta$ satisfies the KKT conditions of Problem~(\ref{eq:optimization problem}).
Thus,
there are $\lambda_1,\ldots,\lambda_n$ such that for every $j \in J$ we have
\begin{equation}
\label{eq:kkt condition w}
	w_j = \sum_{i \in I} \lambda_i \frac{\partial}{\partial w_j} \left( y_i \cn_{\btheta}(x_i) \right) =  \sum_{i \in I} \lambda_i y_i v_j \sigma'_{i,j} x_i~,
\end{equation}
where $\sigma'_{i,j}$ is a subgradient of $\sigma$ at $w_j \cdot x_i + b_j$, i.e., if $w_j \cdot x_i + b_j \neq 0$ then $\sigma'_{i,j} = \onefunc[w_j \cdot x_i + b_j > 0]$, and otherwise $\sigma'_{i,j}$ is some value in $[0,1]$ (we emphasize that in this case $\sigma'_{i,j}$ may be any value in $[0,1]$ and in this proof we do not have any further assumptions on it).  Also we have $\lambda_i \geq 0$ for all $i \in I$, and $\lambda_i=0$ if 
$i \not \in I'$.
Likewise, we have
\begin{equation}
\label{eq:kkt condition b}
	b_j = \sum_{i \in I} \lambda_i \frac{\partial}{\partial b_j} \left( y_i \cn_{\btheta}(x_i) \right) =  \sum_{i \in I} \lambda_i y_i v_j \sigma'_{i,j}~.
\end{equation}

\stam{
Let $1 \leq a < b \leq n$ be indices such that for every $a \leq i,i' \leq b$ we have $y_i=y_{i'}$. Thus the labels do not switch signs for the inputs $x_a, x_{a+1},\ldots,x_b$. 
Intuitively, the proof follows by showing that in the interval $[x_a,x_b]$ the network $\cn_\btheta$ has a constant number of linear regions, and then concluding that the overall number of linear regions in  $\cn_\btheta$ must be $\co(r)$.
We assume here that for all $a \leq i \leq b$ we have $y_i=1$. The proof for the case where the labels are $-1$ should be similar. Also, for simplicity we assume here that $x_i>0$ for all $i \in I$. The proof for the general case should be similar.
}

We say that $\cn_\btheta$ has an \emph{activation point} at $x$ if there is $j \in [J]$ with $w_j \neq 0$ such that $w_j \cdot x + b_j = 0$. In this case we say that  \emph{the activation point $x$ corresponds to the neuron $j$}.
Note that if $w_j=0$ then the neuron computes a constant function and thus it does not affect the number of linear regions in $\cn_\btheta$.

\begin{lemma} \label{lem:bound for each interval}
	We denote $I' = \{i_1,\ldots,i_q\}$ where $1 \leq i_1 < \ldots < i_q \leq [n]$.
	For every $\ell \in [q-1]$ the network $\cn_\btheta$ has at most two activation points in the open interval $(x_{i_\ell},x_{i_{\ell+1}})$.
	Moreover, $\cn_\btheta$ has at most one activation point in $(-\infty,i_1)$ and at most one activation points in $(i_q,\infty)$.
\end{lemma}
\begin{proof}
	Let $x \in (x_{i_\ell},x_{i_{\ell+1}})$ be an activation point, and let $j \in J$ such that $w_j \neq 0$ and $w_j \cdot x + b_j = 0$. 

	Suppose first that $w_j > 0$. Since $w_j \cdot x + b_j = 0$ then for every $x' > x$ we have  $w_j \cdot x' + b_j > 0$, and for every $x' < x$ we have $w_j \cdot x' + b_j < 0$. By \eqref{eq:kkt condition w} we have
	\begin{align*}
		w_j 
		=  \sum_{i \in I} \lambda_i y_i v_j \sigma'_{i,j} x_i
		= \sum_{i \in I'} \lambda_i y_i v_j \sigma'_{i,j} x_i~.
	\end{align*}
	Since $x \neq x_i$ for all $i \in I'$ then $w_j \cdot x_i + b_j \neq 0$, and we have $\sigma'_{i,j} = \onefunc(w_j \cdot x_i + b_j > 0) = \onefunc[x_i > x]$. Therefore, the above displayed equation equals
	\begin{equation} \label{eq:one per interval pos wj}
		 \sum_{i \in I'} \lambda_i y_i v_j \onefunc[x_i > x] x_i
		 = \sum_{i \in I',\, i \geq i_{\ell+1}} \lambda_i y_i v_j x_i~.
	\end{equation}
	Likewise, by \eqref{eq:kkt condition b} we have 
	\begin{equation} \label{eq:one per interval pos bj}
		b_j
		=  \sum_{i \in I} \lambda_i y_i v_j \sigma'_{i,j}
		= \sum_{i \in I'} \lambda_i y_i v_j \sigma'_{i,j} 
		=  \sum_{i \in I'} \lambda_i y_i v_j \onefunc[x_i > x]
		=  \sum_{i \in I',\, i \geq i_{\ell+1}} \lambda_i y_i v_j~.
	\end{equation}	
	By  \eqref{eq:one per interval pos wj} and \eqref{eq:one per interval pos bj}, the activation point $x$ satisfies	
	\[
		x
		= \frac{-b_j}{w_j}
		= \frac{- \sum_{i \in I',\, i \geq i_{\ell+1}} \lambda_i y_i v_j}{ \sum_{i \in I',\, i \geq i_{\ell+1}}\lambda_i y_i v_j x_i}
		=  \frac{- \sum_{i \in I',\, i \geq i_{\ell+1}} \lambda_i y_i}{ \sum_{i \in I',\, i \geq i_{\ell+1}} \lambda_i y_i x_i}~.
	\] 
	Therefore if $x$ and $x'$ are two activation points in $(x_{i_\ell},x_{i_{\ell+1}})$ that correspond to $w_j>0$ and $w_{j'}>0$ respectively, then $x=x'$. Thus, there is at most one activation point $x \in (x_{i_\ell},x_{i_{\ell+1}})$ that corresponds to some $w_j > 0$.
	
	Moreover, by similar arguments there is at most one activation point $x \in (x_{i_\ell},x_{i_{\ell+1}})$ that corresponds to $w_j < 0$. Overall, in the interval $(x_{i_\ell},x_{i_{\ell+1}})$ there are at most two activation points.
	
	If $x \in (-\infty,i_1)$ then from similar argument we get that there is at most one activation point that corresponds to a neuron $j$ with $w_j>0$. Also, an activation point in $ (-\infty,i_1)$ that corresponds to a neuron with $w_j<0$ does not exist, since such neuron is not active for any input $x_i$ with $i \in I'$, and hence by \eqref{eq:kkt condition w} we must have $w_j=0$. The proof of the claim for the interval $(i_q,\infty)$ is similar.
\end{proof}

Let $1 \leq a < b \leq n$ be indices such that for every $a \leq i < i' \leq b$ we have $y_i=y_{i'}$. Thus the labels do not switch signs for the inputs $x_a, x_{a+1},\ldots,x_b$. 
Intuitively, the proof follows by showing that in the interval $[x_a,x_b]$ the network $\cn_\btheta$ has a constant number of linear regions, and then concluding that the overall number of linear regions in  $\cn_\btheta$ must be $\co(r)$.
We first consider the case where $x_b > x_a \geq 0$ and for all $a \leq i \leq b$ we have $y_i=1$. 
In the following lemmas we analyze the activation points in this case and obtain a bound on the number of linear regions. Then, we will extend this result also to the cases where $y_i=-1$ and where $x_a < x_b \leq 0$.
For a given activation point $x$ we say that the derivative of the network increases (respectively, decreases) in $x$ if for every sufficiently small $\varepsilon>0$ the derivative of the network at $x-\varepsilon$ is smaller (respectively, larger) than the derivative at $x+\varepsilon$.

\begin{lemma} \label{lem:bound decreasing activation points}
	Suppose that $x_b > x_a \geq 0$ and for all $a \leq i \leq b$ we have $y_i=1$.
	In the interval $[x_a,x_b]$ the network $\cn_\btheta$ has at most two activation points where the derivative decreases.
\end{lemma}
\begin{proof}
	If the derivative decreases at an activation point $x \geq 0$, then there is at least one neuron $j \in J$ where $w_j \cdot x + b_j = 0$ and the derivative (of the function computed by this neuron) decreases in $x$. There are two types of such neurons: either (a) $w_j > 0$, $b_j \leq 0$ and $v_j < 0$; or (b) $w_j < 0$, $b_j \geq 0$ and $v_j < 0$. 
	
	We now show that in the interval $[x_a,x_b]$ there is at most one activation point that corresponds to a neuron of type (a) and at most one activation point that corresponds to a neuron of type (b). We note that an activation point might correspond to multiple neurons, namely, to a set $Q \subseteq J$ of neurons of size larger than $1$. However, we show that if $x,x'$ are activation points in $[x_a,x_b]$ that correspond to sets $Q_x$ and $Q_{x'}$ of neurons (respectively) and both sets $Q_x,Q_{x'}$ contain neurons of type (a) then $x=x'$. Likewise, if both sets $Q_x,Q_{x'}$ contain neurons of type (b) then we also have $x=x'$.
	
	Suppose towards contradiction that $x \in [x_a,x_b]$ is an activation point that corresponds to a neuron $j$ of type (a), and $x' \in [x_a,x_b]$ is an activation point with $x' > x$ that corresponds to a neuron $j'$ of type (a). Since both neurons $j,j'$ are of type (a), then we have $w_j \cdot z + b_j > 0$ iff $z > x$ and $w_{j'} \cdot z + b_{j'} > 0$ iff $z > x'$. 
	By \eqref{eq:kkt condition w} we have
	\begin{align*}
		\frac{1}{v_{j'}} \cdot w_{j'} 
		= \frac{1}{v_{j'}} \left( \sum_{i \in I} \lambda_i y_i v_{j'} \sigma'_{i,j'} x_i \right)
		= \sum_{i \in I} \lambda_i y_i  \sigma'_{i,j'} x_i 
		\leq \sum_{i \in I} \lambda_i y_i  \onefunc[x_i \geq x'] x_i~,
	\end{align*}
	where the last inequality is since $\sigma'_{i,j'} =  \onefunc[x_i \geq x']$ if $x_i \neq x'$, and when $x_i=x'$ we have $\sigma'_{i,j'} \leq \onefunc[x_i \geq x']$ (and recall that for $a \leq i \leq b$ we have $y_i=1$, $x_i \geq 0$ and $\lambda_i \geq 0$).
	The above RHS equals
	\begin{align*}
		\sum_{i \in I} \lambda_i y_i  \onefunc[x_i > x] x_i  - \sum_{i \in I} \lambda_i y_i  \onefunc[x < x_i < x'] x_i
		&\leq \sum_{i \in I} \lambda_i y_i  \onefunc[x_i > x] x_i
		\\
		&\leq \sum_{i \in I} \lambda_i y_i  \sigma'_{i,j} x_i  
		\\
		&= \frac{1}{v_j} \left(\sum_{i \in I} \lambda_i y_i v_j \sigma'_{i,j} x_i  \right)
		\\
		&= \frac{1}{v_j} \cdot w_j~.
	\end{align*}
	Since $v_j < 0$ we conclude that 
	\begin{equation} \label{eq:wj vs wj'}
		w_j \leq \frac{v_j}{v_{j'}} \cdot w_{j'}~.
	\end{equation}
	
	Likewise, by \eqref{eq:kkt condition b} we have
	\begin{align*}
		\frac{1}{v_{j'}} \cdot b_{j'}
		&=  \frac{1}{v_{j'}} \left(  \sum_{i \in I} \lambda_i y_i v_{j'} \sigma'_{i,j'}\right)
		= \sum_{i \in I} \lambda_i y_i \sigma'_{i,j'}
		\leq \sum_{i \in I} \lambda_i y_i  \onefunc[x_i \geq x']
		\\ 
		&= \sum_{i \in I} \lambda_i y_i \onefunc[x_i > x] - \sum_{i \in I} \lambda_i y_i  \onefunc[x < x_i < x'] 
		\\
		&\leq \sum_{i \in I} \lambda_i y_i  \sigma'_{i,j}   
		= \frac{1}{v_j} \left(\sum_{i \in I} \lambda_i y_i v_j \sigma'_{i,j} \right)
		= \frac{1}{v_j} \cdot b_j~.
	\end{align*}
	Hence, we conclude that 
	\begin{equation} \label{eq:bj vs bj'}
		b_j \leq \frac{v_j}{v_{j'}} \cdot b_{j'}~.
	\end{equation}
	
	Since $0 \leq x < x'$, then by using \eqref{eq:wj vs wj'} and~(\ref{eq:bj vs bj'}) we have 
	\begin{align*}
		0
		< w_{j} \cdot x' + b_{j} 
		\leq \frac{v_j}{v_{j'}} \cdot w_{j'} \cdot x' + \frac{v_j}{v_{j'}} \cdot b_{j'}
		= \frac{v_j}{v_{j'}} \left( w_{j'} \cdot x' +  b_{j'} \right)
		= 0~.
	\end{align*} 
	Thus, we reached a contradiction.
	
	Next, suppose that $x \in [x_a,x_b]$ is an activation point that corresponds to a neuron $j$ of type (b), and $x' \in [x_a,x_b]$ is an activation point with $x' > x$ that corresponds to a neuron $j'$ of type (b). 
	We will reach a contradiction using similar arguments to the case of type (a) neurons, with some required modifications.

	Since both neurons $j,j'$ are of type (b), then we have $w_j \cdot z + b_j > 0$ iff $z < x$ and $w_{j'} \cdot z + b_{j'} > 0$ iff $z < x'$. 
	By \eqref{eq:kkt condition w} we have
	\begin{align*}
		\frac{1}{v_{j'}} \cdot w_{j'} 
		= \frac{1}{v_{j'}} \left( \sum_{i \in I} \lambda_i y_i v_{j'} \sigma'_{i,j'} x_i \right)
		= \sum_{i \in I} \lambda_i y_i \sigma'_{i,j'} x_i 
		\geq \sum_{i \in I} \lambda_i y_i  \onefunc[x_i < x'] x_i~,
	\end{align*}
	where the last inequality is since $\sigma'_{i,j'} =  \onefunc[x_i < x']$ if $x_i \neq x'$, and when $x_i=x'$ we have $\sigma'_{i,j'} \geq \onefunc[x_i < x']$.
	The above RHS equals
	\begin{align*}
		\sum_{i \in I} \lambda_i y_i  \onefunc[x_i \leq x] x_i + \sum_{i \in I} \lambda_i y_i  \onefunc[x < x_i < x'] x_i
		&\geq \sum_{i \in I} \lambda_i y_i  \onefunc[x_i \leq x] x_i 
		\\
		&\geq \sum_{i \in I} \lambda_i y_i  \sigma'_{i,j} x_i  
		\\
		&= \frac{1}{v_j} \left(\sum_{i \in I} \lambda_i y_i v_j \sigma'_{i,j} x_i  \right)
		\\
		&= \frac{1}{v_j} \cdot w_j~.
	\end{align*}
	Since $v_j < 0$ we conclude that 
	\begin{equation} \label{eq:wj vs wj' 2}
		w_j \geq \frac{v_j}{v_{j'}} \cdot w_{j'}~.
	\end{equation}
	
	Likewise, by \eqref{eq:kkt condition b} we have
	\begin{align*}
		\frac{1}{v_{j'}} \cdot b_{j'}
		&=  \frac{1}{v_{j'}} \left(  \sum_{i \in I} \lambda_i y_i v_{j'} \sigma'_{i,j'}\right)
		= \sum_{i \in I} \lambda_i y_i \sigma'_{i,j'}
		\geq \sum_{i \in I} \lambda_i y_i  \onefunc[x_i < x']
		\\ 
		&= \sum_{i \in I} \lambda_i y_i \onefunc[x_i \leq x] + \sum_{i \in I} \lambda_i y_i \onefunc[x < x_i < x']
		\\
		&\geq \sum_{i \in I} \lambda_i y_i \sigma'_{i,j}   
		= \frac{1}{v_j} \left(\sum_{i \in I} \lambda_i y_i v_j \sigma'_{i,j} \right)
		= \frac{1}{v_j} \cdot b_j~.
	\end{align*}
	Hence, we conclude that 
	\begin{equation} \label{eq:bj vs bj' 2}
		b_j \geq \frac{v_j}{v_{j'}} \cdot b_{j'}~.
	\end{equation}
	
	Since $0 \leq x < x'$, and by using \eqref{eq:wj vs wj' 2} and~(\ref{eq:bj vs bj' 2}), we have 
	\begin{align*}
		0
		> w_{j} \cdot x' + b_{j} 
		\geq \frac{v_j}{v_{j'}} \cdot w_{j'} \cdot x' + \frac{v_j}{v_{j'}} \cdot b_{j'}
		= \frac{v_j}{v_{j'}} \left( w_{j'} \cdot x' +  b_{j'} \right)
		= 0~.
	\end{align*} 
	Thus, we reached a contradiction.
\end{proof}

We denote $\ci_{a,b} := \{a,a+1,\ldots,b\} \subseteq I$ and $\ci'_{a,b} := \{i \in \ci_{a,b} : y_i \cn_\btheta(x_i) = 1 \} = \ci_{a,b} \cap I'$. Thus, $\ci'_{a,b}$ are the indices of the examples in the interval $[x_a,x_b]$ where $\cn_\btheta$ attains margin of exactly $1$. 
We denote $\ci'_{a,b} = \{i_1,\ldots,i_m\}$, where $a \leq i_1 < \ldots < i_m \leq b$. 

\begin{lemma} \label{lem:at most two upper corners}
	Suppose that $x_b > x_a \geq 0$ and for all $a \leq i \leq b$ we have $y_i=1$.
	There are at most $2$ indices $\ell \in [m-1]$ such that $\cn_\btheta(x) > 1$ for some $x \in [x_{i_\ell},x_{i_{\ell+1}}]$.
\end{lemma}
\begin{proof}
	Assume that $\cn_\btheta(x) > 1$ for $x \in [x_{i_\ell},x_{i_{\ell+1}}]$. Since $\cn_\btheta(x_{i_\ell}) = \cn_\btheta(x_{i_{\ell+1}}) = 1$, then we have $x \in (x_{i_\ell},x_{i_{\ell+1}})$. Now, since $\cn_\btheta(x_{i_\ell}) = \cn_\btheta(x_{i_{\ell+1}}) = 1$ and $\cn_\btheta(x) > 1$ for some $x \in (x_{i_\ell},x_{i_{\ell+1}})$, then there must be an activation point in $(x_{i_\ell},x_{i_{\ell+1}})$ where the derivative decreases. Since by \lemref{lem:bound decreasing activation points} there are at most two such activation points in $[x_a,x_b]$ then the lemma follows.
\end{proof}

\begin{lemma} \label{lem:lower corners}
	Suppose that $x_b > x_a \geq 0$ and for all $a \leq i \leq b$ we have $y_i=1$.
	There are at most $5$ indices $\ell \in [m-1]$ such that $\cn_\btheta(x) < 1$ for some $x \in [x_{i_\ell},x_{i_{\ell+1}}]$. 
\end{lemma}
\begin{proof}
	Assume that $\cn_\btheta(x) < 1$ for $x \in [x_{i_\ell},x_{i_{\ell+1}}]$. Since $\cn_\btheta(x_{i_\ell}) = \cn_\btheta(x_{i_{\ell+1}}) = 1$, then $x \in (x_{i_\ell},x_{i_{\ell+1}})$. 
	If $\ell \neq m-1$ then we have $\cn_\btheta(x) < 1$ and $\cn_\btheta(x_{i_{\ell+1}}) = \cn_\btheta(x_{i_{\ell+2}}) = 1$. Hence, the interval $(x,x_{i_{\ell+1}})$ contains a point with positive derivative, and the interval $(x_{i_{\ell+1}}, x_{i_{\ell+2}})$ contains a point with non-positive derivative. Thus, there must be an activation point in $(x,x_{i_{\ell+2}})$ where the derivative decreases. Therefore, there is an activation point with decreasing derivative either in the interval $(x_{i_\ell},x_{i_{\ell+1}}]$ or in the interval $[x_{i_{\ell+1}},x_{i_{\ell+2}})$ (and possibly in both). Thus, an interval  $[x_{i_\ell},x_{i_{\ell+1}}]$ for $\ell \neq m-1$ might contain some $x$ with $\cn_\btheta(x) < 1$ only if there is an activation point with decreasing derivative in $(x_{i_\ell},x_{i_{\ell+1}}]$ or $[x_{i_{\ell+1}},x_{i_{\ell+2}})$. Since by \lemref{lem:bound decreasing activation points} there are at most two such activation points in $[x_a,x_b]$, then there are at most $4$ intervals $[x_{i_\ell},x_{i_{\ell+1}}]$ with $\ell \neq m-1$ that contain some $x$ with $\cn_\btheta(x) < 1$. The interval $[x_{i_{m-1}},x_{i_{m}}]$ might also contain such $x$. Overall, there are at most $5$ indices $\ell \in [m-1]$ such that $\cn_\btheta(x) < 1$ for some $x \in [x_{i_\ell},x_{i_{\ell+1}}]$.
\end{proof}

\begin{lemma} \label{lem:boundaries 1 to m}
	Suppose that $x_b > x_a \geq 0$ and for all $a \leq i \leq b$ we have $y_i=1$.
	There are at most $30$ boundaries between linear regions in $[x_{i_1},x_{i_m}]$.
\end{lemma}
\begin{proof}
 By Lemmas~\ref{lem:at most two upper corners} and~\ref{lem:lower corners} there are at most $7$ indices $\ell \in [m-1]$ such that $\cn_\btheta(x) \neq 1$ for some $x \in [x_{i_\ell},x_{i_{\ell+1}}]$. We denote the set of these indices by $R$. Let $x \in (x_{i_1},x_{i_m})$ be a boundary between two linear regions. Note that if $x \in (x_{i_\ell},x_{i_{\ell+1}})$ for some $\ell \in [m-1]$ then $\ell \in R$. Also, if $x = x_{i_\ell}$ for some $2 \leq \ell \leq m-1$ then either $\ell \in R$ or $\ell-1 \in R$. In any case, we have $x \in [x_{i_\ell},x_{i_{\ell+1}}]$ for some $\ell \in R$. Note that each boundary between linear regions is also an activation point. Therefore, the number of boundaries between linear regions in $(x_{i_1},x_{i_m})$ is at most the number of activation points in the intervals $[x_{i_\ell},x_{i_{\ell+1}}]$ with $\ell \in R$. By \lemref{lem:bound for each interval} each interval $[x_{i_\ell},x_{i_{\ell+1}}]$ contains at most $4$ activation points: two points in $(x_{i_\ell},x_{i_{\ell+1}})$ and two in $\{x_{i_\ell},x_{i_{\ell+1}}\}$. Overall, there are at most $|R| \cdot 4 \leq 28$ boundaries between linear regions in $(x_{i_1},x_{i_m})$. Thus, there are at most $30$ boundaries between linear regions in $[x_{i_1},x_{i_m}]$. 
\end{proof}

In the above lemmas we considered the case where $x_b > x_a \geq 0$ and for all $a \leq i \leq b$ we have $y_i=1$, and proved that $\ci'_{a,b}$ is such that there are at most $30$ boundaries between linear regions in $[x_{i_1},x_{i_m}]$.
In \subsecref{app:missing lemmas} we show analogous results for the case where $x_b > x_a \geq 0$ and for all $a \leq i \leq b$ we have $y_i=-1$. Thus, if $x_b > x_a \geq 0$ and the labels do not switch sign in the interval $[x_a,x_b]$ (i.e., either all labels are $1$ or all labels are $-1$) then there are at most $30$ boundaries between linear regions in $[x_{i_1},x_{i_m}]$.
The case where $x_a < x_b \leq 0$ (and the labels do not switch sign in the interval $[x_a,x_b]$) can be handled in a similar manner. Thus, even where the inputs are negative, $\ci'_{a,b}$ is such that there are at most $30$ boundaries between linear regions in $[x_{i_1},x_{i_m}]$. The proof for this case is similar and for conciseness we do not repeat it.

We are now ready to finish the proof of the theorem.
Consider the set $I'$ of indices where $\cn_\btheta$ attains margin $1$ and denote $I' = \{i_1,\ldots,i_q\}$.
Note that if $I'$ is an empty set, then by \eqref{eq:kkt condition w} and~(\ref{eq:kkt condition b}) all neurons have $w_j=b_j=0$ and hence the network $\cn_\btheta$ is the zero function.
Let $\ell \leq \ell'$ be such that the labels of the examples in the dataset do not change sign in the interval $[x_{i_\ell},x_{i_{\ell'}}]$, and either $0 \leq x_{i_\ell} \leq x_{i_{\ell'}}$ or $x_{i_\ell} \leq x_{i_{\ell'}} \leq 0$. Thus, the interval $[x_{i_{\ell}},x_{i_{\ell'}}]$ contains at most $30$ boundaries between linear regions. Also, by \lemref{lem:bound for each interval} the interval $(x_{i_{\ell-1}},x_{i_{\ell}})$ (or $(-\infty,x_{i_{\ell}})$ if $\ell=1$) contains at most two boundaries between linear regions. Likewise, the interval $(x_{i_{\ell'}},x_{i_{\ell'+1}})$ (or $(x_{i_{\ell'}},\infty)$ if $\ell'=q$) contains at most two boundaries between linear regions. Recall that the labels in the dataset switch sign at most $r$ times. Overall, we get that the number of boundaries between linear regions in the whole domain $\reals$ is at most 
$30(r+2)+2(r+3) = 32 r + 66$. Indeed, if one of the $r+1$ intervals where $\cn_\btheta$ do not switch sign contains $0$ then we split it into two intervals, and thus we obtain $r+2$ intervals. Each of these intervals includes at most $30$ boundaries, and outside of these intervals there are at most $2(r+3)$ boundaries. Thus, that are at most $32 r + 67$ linear regions.

\subsection{Lemmas for the case $y_i=-1$} \label{app:missing lemmas}

\begin{lemma} \label{lem:bound increasing activation points}
	Suppose that $x_b > x_a \geq 0$ and for all $a \leq i \leq b$ we have $y_i=-1$.
	In the interval $[x_a,x_b]$ the network $\cn_\btheta$ has at most two activation points where the derivative increases.
\end{lemma}
\begin{proof}
	If the derivative increases at an activation point $x \geq 0$, then there is a least one neuron $j \in J$ where $w_j \cdot x + b_j = 0$ and the derivative (of the function computed by this neuron) increases in $x$. There are two types of such neurons: either (a) $w_j > 0$, $b_j \leq 0$ and $v_j > 0$; or (b) $w_j < 0$, $b_j \geq 0$ and $v_j > 0$. 
	
	We now show that in the interval $[x_a,x_b]$ there is at most one activation point that corresponds to a neuron of type (a) and at most one activation point that corresponds to a neuron of type (b). We note that an activation point might correspond to multiple neurons, namely, to a set $Q \subseteq J$ of neurons of size larger than $1$. However, we show that if $x,x'$ are activation points in $[x_a,x_b]$ that correspond to sets $Q_x$ and $Q_{x'}$ of neurons (respectively) and both sets $Q_x,Q_{x'}$ contain neurons of type (a) then $x=x'$. Likewise, if both sets $Q_x,Q_{x'}$ contain neurons of type (b) then we also have $x=x'$.
	
	Suppose towards contradiction that $x \in [x_a,x_b]$ is an activation point that corresponds to a neuron $j$ of type (a), and $x' \in [x_a,x_b]$ is an activation point with $x' > x$ that corresponds to a neuron $j'$ of type (a). Since both neurons $j,j'$ are of type (a), then we have $w_j \cdot z + b_j > 0$ iff $z > x$ and $w_{j'} \cdot z + b_{j'} > 0$ iff $z > x'$. 
	By \eqref{eq:kkt condition w} we have
	\begin{align*}
		\frac{1}{v_{j'}} \cdot w_{j'} 
		= \frac{1}{v_{j'}} \left( \sum_{i \in I} \lambda_i y_i v_{j'} \sigma'_{i,j'} x_i \right)
		= \sum_{i \in I} \lambda_i y_i  \sigma'_{i,j'} x_i 
		\geq \sum_{i \in I} \lambda_i y_i  \onefunc[x_i \geq x'] x_i~,
	\end{align*}
	where the last inequality is since $\sigma'_{i,j'} =  \onefunc[x_i \geq x']$ if $x_i \neq x'$, and when $x_i=x'$ we have $\sigma'_{i,j'} \leq \onefunc[x_i \geq x']$ (and recall that for $a \leq i \leq b$ we have $y_i=-1$, $x_i \geq 0$ and $\lambda_i \geq 0$).
	The above RHS equals
	\begin{align*}
		\sum_{i \in I} \lambda_i y_i  \onefunc[x_i > x] x_i  - \sum_{i \in I} \lambda_i y_i  \onefunc[x < x_i < x'] x_i
		&\geq \sum_{i \in I} \lambda_i y_i  \onefunc[x_i > x] x_i
		\\
		&\geq \sum_{i \in I} \lambda_i y_i  \sigma'_{i,j} x_i  
		\\
		&= \frac{1}{v_j} \left(\sum_{i \in I} \lambda_i y_i v_j \sigma'_{i,j} x_i  \right)
		\\
		&= \frac{1}{v_j} \cdot w_j~.
	\end{align*}
	We conclude that 
	\begin{equation} \label{eq:wj vs wj' negative y}
		w_j \leq \frac{v_j}{v_{j'}} \cdot w_{j'}~.
	\end{equation}
	
	Likewise, by \eqref{eq:kkt condition b} we have
	\begin{align*}
		\frac{1}{v_{j'}} \cdot b_{j'}
		&=  \frac{1}{v_{j'}} \left(  \sum_{i \in I} \lambda_i y_i v_{j'} \sigma'_{i,j'}\right)
		= \sum_{i \in I} \lambda_i y_i \sigma'_{i,j'}
		\geq \sum_{i \in I} \lambda_i y_i  \onefunc[x_i \geq x']
		\\ 
		&= \sum_{i \in I} \lambda_i y_i \onefunc[x_i > x] - \sum_{i \in I} \lambda_i y_i  \onefunc[x < x_i < x'] 
		\\
		&\geq \sum_{i \in I} \lambda_i y_i  \sigma'_{i,j}   
		= \frac{1}{v_j} \left(\sum_{i \in I} \lambda_i y_i v_j \sigma'_{i,j} \right)
		= \frac{1}{v_j} \cdot b_j~.
	\end{align*}
	Hence, we conclude that 
	\begin{equation} \label{eq:bj vs bj' negative y}
		b_j \leq \frac{v_j}{v_{j'}} \cdot b_{j'}~.
	\end{equation}
	
	Since $0 \leq x < x'$, then by using \eqref{eq:wj vs wj' negative y} and~(\ref{eq:bj vs bj' negative y}) we have 
	\begin{align*}
		0
		< w_{j} \cdot x' + b_{j} 
		\leq \frac{v_j}{v_{j'}} \cdot w_{j'} \cdot x' + \frac{v_j}{v_{j'}} \cdot b_{j'}
		= \frac{v_j}{v_{j'}} \left( w_{j'} \cdot x' +  b_{j'} \right)
		= 0~.
	\end{align*} 
	Thus, we reached a contradiction.
	
	Next, suppose that $x \in [x_a,x_b]$ is an activation point that corresponds to a neuron $j$ of type (b), and $x' \in [x_a,x_b]$ is an activation point with $x' > x$ that corresponds to a neuron $j'$ of type (b). 
	We will reach a contradiction using similar arguments to the case of type (a) neurons, with some required modifications.

	Since both neurons $j,j'$ are of type (b), then we have $w_j \cdot z + b_j > 0$ iff $z < x$ and $w_{j'} \cdot z + b_{j'} > 0$ iff $z < x'$. 
	By \eqref{eq:kkt condition w} we have
	\begin{align*}
		\frac{1}{v_{j'}} \cdot w_{j'} 
		= \frac{1}{v_{j'}} \left( \sum_{i \in I} \lambda_i y_i v_{j'} \sigma'_{i,j'} x_i \right)
		= \sum_{i \in I} \lambda_i y_i \sigma'_{i,j'} x_i 
		\leq \sum_{i \in I} \lambda_i y_i  \onefunc[x_i < x'] x_i~,
	\end{align*}
	where the last inequality is since $\sigma'_{i,j'} =  \onefunc[x_i < x']$ if $x_i \neq x'$, and when $x_i=x'$ we have $\sigma'_{i,j'} \geq \onefunc[x_i < x']$ (and $y_i=-1$).
	The above RHS equals
	\begin{align*}
		\sum_{i \in I} \lambda_i y_i  \onefunc[x_i \leq x] x_i + \sum_{i \in I} \lambda_i y_i  \onefunc[x < x_i < x'] x_i
		&\leq \sum_{i \in I} \lambda_i y_i  \onefunc[x_i \leq x] x_i 
		\\
		&\leq \sum_{i \in I} \lambda_i y_i  \sigma'_{i,j} x_i  
		\\
		&= \frac{1}{v_j} \left(\sum_{i \in I} \lambda_i y_i v_j \sigma'_{i,j} x_i  \right)
		\\
		&= \frac{1}{v_j} \cdot w_j~.
	\end{align*}
	We conclude that 
	\begin{equation} \label{eq:wj vs wj' 2 negative y}
		w_j \geq \frac{v_j}{v_{j'}} \cdot w_{j'}~.
	\end{equation}
	
	Likewise, by \eqref{eq:kkt condition b} we have
	\begin{align*}
		\frac{1}{v_{j'}} \cdot b_{j'}
		&=  \frac{1}{v_{j'}} \left(  \sum_{i \in I} \lambda_i y_i v_{j'} \sigma'_{i,j'}\right)
		= \sum_{i \in I} \lambda_i y_i \sigma'_{i,j'}
		\leq \sum_{i \in I} \lambda_i y_i  \onefunc[x_i < x']
		\\ 
		&= \sum_{i \in I} \lambda_i y_i \onefunc[x_i \leq x] + \sum_{i \in I} \lambda_i y_i \onefunc[x < x_i < x']
		\\
		&\leq \sum_{i \in I} \lambda_i y_i \sigma'_{i,j}   
		= \frac{1}{v_j} \left(\sum_{i \in I} \lambda_i y_i v_j \sigma'_{i,j} \right)
		= \frac{1}{v_j} \cdot b_j~.
	\end{align*}
	Hence, we conclude that 
	\begin{equation} \label{eq:bj vs bj' 2 negative y}
		b_j \geq \frac{v_j}{v_{j'}} \cdot b_{j'}~.
	\end{equation}
	
	Since $0 \leq x < x'$, and by using \eqref{eq:wj vs wj' 2 negative y} and~(\ref{eq:bj vs bj' 2 negative y}), we have 
	\begin{align*}
		0
		> w_{j} \cdot x' + b_{j} 
		\geq \frac{v_j}{v_{j'}} \cdot w_{j'} \cdot x' + \frac{v_j}{v_{j'}} \cdot b_{j'}
		= \frac{v_j}{v_{j'}} \left( w_{j'} \cdot x' +  b_{j'} \right)
		= 0~.
	\end{align*} 
	Thus, we reached a contradiction.
\end{proof}

We use the notations $\ci_{a,b} := \{a,a+1,\ldots,b\} \subseteq I$ and $\ci'_{a,b} := \{i \in \ci_{a,b} : y_i \cn_\btheta(x_i) = 1 \} = \ci \cap I'$. Thus, $\ci'$ are the indices of the examples in the interval $[x_a,x_b]$ where $\cn_\btheta$ attains margin of exactly $1$. 
We denote $\ci'_{a,b} = \{i_1,\ldots,i_m\}$, where $a \leq i_1 < \ldots < i_m \leq b$. 

\begin{lemma} \label{lem:at most two lower corners}
	Suppose that $x_b > x_a \geq 0$ and for all $a \leq i \leq b$ we have $y_i=-1$.
	There are at most $2$ indices $\ell \in [m-1]$ such that $\cn_\btheta(x) < -1$ for some $x \in [x_{i_\ell},x_{i_{\ell+1}}]$.
\end{lemma}
\begin{proof}
	Assume that $\cn_\btheta(x) < -1$ for $x \in [x_{i_\ell},x_{i_{\ell+1}}]$. Since $\cn_\btheta(x_{i_\ell}) = \cn_\btheta(x_{i_{\ell+1}}) = -1$, then we have $x \in (x_{i_\ell},x_{i_{\ell+1}})$. Now, since $\cn_\btheta(x_{i_\ell}) = \cn_\btheta(x_{i_{\ell+1}}) = -1$ and $\cn_\btheta(x) < -1$ for some $x \in (x_{i_\ell},x_{i_{\ell+1}})$, then there must be an activation point in $(x_{i_\ell},x_{i_{\ell+1}})$ where the derivative increases. Since by \lemref{lem:bound increasing activation points} there are at most two such activation points in $[x_a,x_b]$ then the lemma follows.
\end{proof}

\begin{lemma} \label{lem:upper corners negative y}
	Suppose that $x_b > x_a \geq 0$ and for all $a \leq i \leq b$ we have $y_i=-1$.
	There are at most $5$ indices $\ell \in [m-1]$ such that $\cn_\btheta(x) > -1$ for some $x \in [x_{i_\ell},x_{i_{\ell+1}}]$. 
\end{lemma}
\begin{proof}
	Assume that $\cn_\btheta(x) > -1$ for $x \in [x_{i_\ell},x_{i_{\ell+1}}]$. Since $\cn_\btheta(x_{i_\ell}) = \cn_\btheta(x_{i_{\ell+1}}) = -1$, then $x \in (x_{i_\ell},x_{i_{\ell+1}})$. 
	If $\ell \neq m-1$ then we have $\cn_\btheta(x) > -1$ and $\cn_\btheta(x_{i_{\ell+1}}) = \cn_\btheta(x_{i_{\ell+2}}) = -1$. Hence, the interval $(x,x_{i_{\ell+1}})$ contains a point with negative derivative, and the interval $(x_{i_{\ell+1}}, x_{i_{\ell+2}})$ contains a point with non-negative derivative. Thus, there must be an activation point in $(x,x_{i_{\ell+2}})$ where the derivative increases. Therefore, there is an activation point with increasing derivative either in the interval $(x_{i_\ell},x_{i_{\ell+1}}]$ or in the interval $[x_{i_{\ell+1}},x_{i_{\ell+2}})$ (and possibly in both). Thus, an interval  $[x_{i_\ell},x_{i_{\ell+1}}]$ for $\ell \neq m-1$ might contain some $x$ with $\cn_\btheta(x) > -1$ only if there is an activation point with increasing derivative in $(x_{i_\ell},x_{i_{\ell+1}}]$ or $[x_{i_{\ell+1}},x_{i_{\ell+2}})$. Since by \lemref{lem:bound increasing activation points} there are at most two such activation points in $[x_a,x_b]$, then there are at most $4$ intervals $[x_{i_\ell},x_{i_{\ell+1}}]$ with $\ell \neq m-1$ that contain some $x$ with $\cn_\btheta(x) > -1$. The interval $[x_{i_{m-1}},x_{i_{m}}]$ might also contain such $x$. Overall, there are at most $5$ indices $\ell \in [m-1]$ such that $\cn_\btheta(x) > -1$ for some $x \in [x_{i_\ell},x_{i_{\ell+1}}]$.
\end{proof}

\begin{lemma} \label{lem:boundaries 1 to m negative y}
	Suppose that $x_b > x_a \geq 0$ and for all $a \leq i \leq b$ we have $y_i=-1$.
	There are at most $30$ boundaries between linear regions in $[x_{i_1},x_{i_m}]$.
\end{lemma}
\begin{proof}
	The proof is similar to the proof of \lemref{lem:boundaries 1 to m}. The only difference is that here we use Lemmas~\ref{lem:bound increasing activation points} and~\ref{lem:at most two lower corners} in order to conclude that there are at most $7$ indices $\ell \in [m-1]$ such that $\cn_\btheta(x) \neq -1$ for some $x \in [x_{i_\ell},x_{i_{\ell+1}}]$, and denote the set of these indices by $R$.
\end{proof}

\end{document}